%% file: main.tex
\documentclass{article}
\usepackage{iclr2027_conference,times}
\usepackage[utf8]{inputenc}
\usepackage[T1]{fontenc}
\usepackage{hyperref}
\usepackage{url}
\usepackage{booktabs}
\usepackage{multirow}
\usepackage{amsmath,amssymb,amsfonts,amsthm,mathtools}
\usepackage{graphicx}
\usepackage{algorithm}
\usepackage{algpseudocode}
\usepackage{aliascnt}
\usepackage[capitalize,noabbrev]{cleveref}
\usepackage{xcolor}

\newcommand{\revision}[1]{#1}

\usepackage{microtype}
\usepackage{anyfontsize}
\usepackage{nicefrac}
\usepackage{longtable}
\usepackage{array}
\usepackage{float}
\usepackage{wrapfig}
\usepackage{capt-of}
\usepackage{needspace}
\usepackage{placeins}

\DeclareRobustCommand{\zs}[1]{}
\newcommand{\lhh}[1]{}
\hypersetup{colorlinks=true, citecolor=blue, linkcolor=blue, urlcolor=blue, hypertexnames=false, pdftitle={When Lower Reconstruction Loss Hurts: Distributionally Robust Refinement for Low-Bit LLM Quantization}, pdfauthor={Yanlong Zhao, Xiaoyuan Cheng, Huihang Liu, Baihua He, Xinyu Zhang, Harrison Bo Hua Zhu, Wenlong Chen, Li Zeng, Zhuo Sun}, pdfsubject={Distributionally robust refinement for weight-only post-training quantization.}}

\definecolor{mockblue}{RGB}{23,107,145}
\definecolor{mockorange}{RGB}{214,139,70}
\definecolor{mockpale}{RGB}{249,243,232}
\definecolor{ProtocolOrange}{HTML}{CC6600}

\DeclareRobustCommand{\langrevision}[1]{{#1}}
\definecolor{SettingRevisionPurple}{HTML}{7A1FA2}
\DeclareRobustCommand{\settingrevision}[1]{{#1}}
\input{tex/commands.tex}

\input{tex/main/00_title_and_authors}

\iclrfinalcopy
\begin{document}
\maketitle
\lhead{Preprint}

\raggedbottom
\setlength{\textfloatsep}{10pt plus 2pt minus 2pt}
\setlength{\floatsep}{8pt plus 2pt minus 2pt}
\setlength{\intextsep}{8pt plus 2pt minus 2pt}
\setlength{\abovecaptionskip}{5pt}
\setlength{\belowcaptionskip}{1pt}

\input{tex/main/00_abstract}

\input{tex/main/01_introduction}
\input{tex/main/03_preliminaries}
\input{tex/main/05_marq}
\input{tex/main/06_experiments}
\input{tex/main/02_related_work}
\input{tex/main/07_conclusion}
\FloatBarrier

\clearpage
\bibliographystyle{iclr2027_conference}
\bibliography{references,marq_references}

\clearpage
\appendix
\input{tex/appendix/00_appendix}

\end{document}

%% file: tex/commands.tex
\newcommand{\R}{\mathbb{R}}
\newcommand{\Tr}{\operatorname{tr}}

\newcommand{\cQ}{\mathcal{Q}}
\newcommand{\cS}{\mathcal{S}}
\newcommand{\cL}{\mathcal{L}}
\newcommand{\cJ}{\mathcal{J}}
\newcommand{\bB}{\mathbf{B}}
\newcommand{\bC}{\mathbf{C}}
\newcommand{\bH}{\mathbf{H}}
\newcommand{\bI}{\mathbf{I}}
\newcommand{\bW}{\mathbf{W}}
\newcommand{\bX}{\mathbf{X}}
\newcommand{\What}{\widehat{\mathbf{W}}}
\newcommand{\norm}[1]{\left\lVert#1\right\rVert}

\DeclareMathOperator*{\argmax}{arg\,max}
\DeclareMathOperator*{\argmin}{arg\,min}

\theoremstyle{plain}
\newtheorem{theorem}{Theorem}[section]
\newaliascnt{proposition}{theorem}
\newtheorem{proposition}[proposition]{Proposition}
\aliascntresetthe{proposition}
\newaliascnt{corollary}{theorem}
\newtheorem{corollary}[corollary]{Corollary}
\aliascntresetthe{corollary}
\newaliascnt{lemma}{theorem}

\aliascntresetthe{lemma}
\theoremstyle{definition}
\newaliascnt{definition}{theorem}

\aliascntresetthe{definition}
\theoremstyle{remark}
\newaliascnt{remark}{theorem}

\aliascntresetthe{remark}
\crefname{proposition}{Proposition}{Propositions}
\Crefname{proposition}{Proposition}{Propositions}
\crefname{corollary}{Corollary}{Corollaries}
\Crefname{corollary}{Corollary}{Corollaries}
\crefname{definition}{Definition}{Definitions}
\Crefname{definition}{Definition}{Definitions}

%% file: tex/main/00_title_and_authors.tex
\title{When Lower Reconstruction Loss Hurts:\\Distributionally Robust Refinement for Low-Bit LLM Quantization}
\author{%
\textbf{Yanlong Zhao}$^{1}$\thanks{Equal Contribution: \href{mailto:zhaoyanlong@mail.ustc.edu.cn}{Yanlong Zhao}, \href{mailto:ucesxc4@ucl.ac.uk}{Xiaoyuan Cheng} and \href{mailto:liuhuihang@mail.sufe.edu.cn}{Huihang Liu}.} \quad
\textbf{Xiaoyuan Cheng}$^{2}$\footnotemark[1] \quad
\textbf{Huihang Liu}$^{3}$\footnotemark[1] \quad
\textbf{Baihua He}$^{1}$ \quad
\textbf{Xinyu Zhang}$^{1,4}$ \\
\textbf{Harrison Bo Hua Zhu}$^{5,6,7}$ \quad
\textbf{Wenlong Chen}$^{6}$ \quad
\textbf{Li Zeng}$^{8}$ \quad
\textbf{Zhuo Sun}$^{3,6}$\thanks{Correspondence Author. Correspondence to \href{mailto:zhuosunreid@outlook.com}{Zhuo Sun}.}\\[0.2cm]
$^{1}$University of Science and Technology of China,
$^{2}$University College London,\\
$^{3}$Shanghai University of Finance and Economics,
$^{4}$AMSS, Chinese Academy of Sciences,\\
$^{5}$University of Copenhagen,
$^{6}$Imperial College London,\\
$^{7}$Technical University of Denmark,
$^{8}$Shenzhen Research Institute of Big Data
}

%% file: tex/main/00_abstract.tex
\begin{abstract}
Weight-only post-training quantization (PTQ) relies heavily on reconstruction loss minimization to preserve model quality at low precision. We show that the weights favored by minimizing this loss need not yield better model performance on new tasks. In fact, we find that lower reconstruction loss can even degrade model performance on the same calibration data. Our analysis further shows that weights with lower reconstruction loss on calibration data can have higher loss than other weights when the distribution of input activations changes. Motivated by these observations and our analysis, we propose Distributionally Robust Quantization (DRQ), a post-hoc refinement process that minimizes worst-case reconstruction loss over a constrained set of input activation distributions. DRQ refines the integer codes representing quantized weights within the existing quantization grid, keeping quantization parameters and inference operators unchanged. Extensive experiments show that DRQ improves models quantized by six representative PTQ methods, including AWQ, GPTQ, and ParoQuant, and delivers gains across both dense and mixture-of-experts large language models. These results establish DRQ as a general post-hoc refinement framework for weight-only PTQ, achieving better downstream performance without adding inference overhead.
\end{abstract}

%% file: tex/main/01_introduction.tex
\section{Introduction}
\label{sec:introduction}

Large language models (LLMs) have demonstrated strong capabilities in language understanding, reasoning, and code generation \citep{gpt4,llama3,humaneval,deepseekai2026deepseekv41flash}. Their increasing memory footprint and bandwidth demands, however, make deployment costly on resource-constrained devices. Weight-only post-training quantization (PTQ) addresses this problem by compressing pretrained weights into low-bit formats such as INT4 while retaining the original activation precision. It reduces weight storage and memory traffic to accelerate memory-bound autoregressive decoding \citep{frantar2023gptq,lin2024awq,Ashkboos2024quarot}.

Most weight-only PTQ methods rely on minimizing reconstruction loss on calibration data to preserve model quality \citep{frantar2023gptq,lin2024awq,omniquant,tseng2024qtip,li2025gptqv2,sun2025flatquant,liang2026paroquant}. For a linear layer, this means matching the output of the quantized weights to that of the full-precision weights on the calibration data,
\begin{equation}
\min_{\What\in\cQ}\;\cL(\What),
\qquad
\cL(\What)=\frac{1}{n}\norm{\What\bX-\bW\bX}_{\rm F}^{2},
\label{eq:intro-reconstruction}
\end{equation}
where $\bW$ and $\What$ denote the full-precision and quantized weights, respectively, the quantization grid $\cQ$ denotes the set of all possible quantized weights, and $\bX$ denotes the associated input activations of the calibration data. Methods differ in how they minimize this loss. For example, GPTQ compensates for rounding errors \citep{frantar2023gptq}, AWQ searches for scaling and clipping parameters \citep{lin2024awq}, and ParoQuant learns transformations \citep{liang2026paroquant}.

We show that reducing reconstruction loss on calibration data can degrade the performance of the quantized model. We first quantize Llama-3.2-1B-Instruct with 3-bit AWQ using WikiText-2 (WT2) \citep{wikitext} as the calibration dataset. Starting from the weights quantized by AWQ, we further optimize one linear layer at a time to reduce reconstruction loss on the calibration data. Specifically, we search over the integer codes representing its weights, while keeping the scales, zero points, and all other layers' weights fixed. We repeat this procedure at 16 layers of the model using two different calibration sets sampled from WT2, producing 32 modified weight matrices. In 18 of 32 cases, the modified weights yield lower reconstruction loss but higher model perplexity (PPL) on the same calibration data (\cref{fig:reconstruction-trap}a). This mismatch also extends to held-out data. In a separate experiment, we examine all 112 linear layers of Llama-3.2-1B-Instruct quantized with 3-bit AWQ. For each layer, we test integer-code replacements while keeping the scales and zero points fixed. In each row of the weight matrix, we retain the single-code change that most reduces reconstruction loss on calibration data. We sort these changes from the largest loss reduction to the smallest. Starting from the original quantized weights each time, we create five modified weight matrices by applying the top $1/256$, $1/128$, $1/64$, $1/32$, and $1/16$ of the changes in this list, respectively. This finally yields 560 modified weight matrices, all with lower reconstruction loss on the calibration data than the corresponding original quantized weights. We then evaluate each matrix by replacing only the corresponding layer's weights and measuring model PPL on WT2 data reserved for evaluation. Of these 560 replacements, 413 (73.8\%) increase PPL relative to the initial quantized model.

\begin{figure*}[t]
  \centering
  \setlength{\parskip}{0pt}
  \setlength{\abovecaptionskip}{-6pt}
  \includegraphics[width=\textwidth]{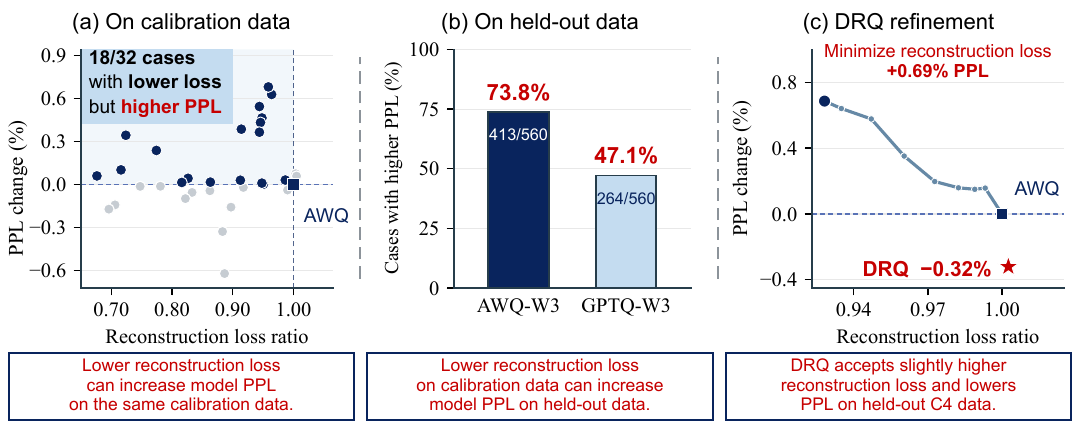}
\caption{Lower reconstruction loss need not improve quantized model performance. Results are from Llama-3.2-1B-Instruct with 3-bit quantization. (a) Among 32 modifications to AWQ-quantized weights, 18 reduce reconstruction loss but increase model PPL on the same WT2 calibration data. (b) Among 560 modified weight matrices per method that reduce reconstruction loss on calibration data, 413 (73.8\%) for AWQ and 264 (47.1\%) for GPTQ increase model PPL on held-out WT2 data. (c) Further minimizing reconstruction loss increases held-out C4 PPL by 0.69\%, whereas DRQ accepts slightly higher reconstruction loss on calibration data and lowers C4 PPL by 0.32\%. Each evaluation modifies one linear layer, keeping all other weights fixed. In (a,c), reconstruction loss is divided by its initial value, and PPL changes are relative to the initial quantized model.}
  \label{fig:reconstruction-trap}
\end{figure*}

These results expose a mismatch between reducing reconstruction loss on calibration data and improving the performance of the quantized model. Ideally, minimizing reconstruction loss brings the outputs of quantized layers closer to those of their full-precision counterparts on calibration data. Yet this closer match need not generalize to other datasets or improve model performance even on the same calibration data. Optimizing reconstruction loss alone can therefore degrade the performance of the quantized model.

To figure out this problem, we study how this loss depends on the input activations observed during calibration. For example, when an input channel of a linear layer stays close to zero, even large differences between the quantized and full-precision weights in the corresponding column can have little effect on the layer outputs. Yet new data could produce larger activations in this channel, so the same weight differences can cause larger output differences and thus higher reconstruction loss. Our analysis identifies when a quantized weight matrix with lower reconstruction loss on calibration data can have higher loss than another under constrained changes in the distribution of input activations (\cref{thm:ordering-certificate}).

Motivated by this analysis, we propose \emph{Distributionally Robust Quantization (DRQ)}, a general refinement framework for weight-only PTQ. DRQ minimizes the worst-case reconstruction loss over a set of input activation distributions, with constraints on how these distributions differ from the calibration distribution. This objective combines reconstruction loss on calibration data with its largest possible increase under these constraints. A real illustrative example in \cref{fig:reconstruction-trap}c shows that further minimizing reconstruction loss on calibration data increases model PPL on held-out C4. DRQ instead accepts a small increase in reconstruction loss on calibration data and lowers PPL on held-out C4 relative to the initial model quantized by AWQ. The resulting quantized weights therefore perform better on data not used for calibration. DRQ updates the integer codes of quantized weights while keeping the bit width, grouping, scales, zero points, and inference operators unchanged. All refinement is offline, so DRQ adds no inference overhead.

\paragraph{Contributions.} We uncover a counterintuitive phenomenon in weight-only PTQ. Lower reconstruction loss on calibration data does not necessarily lead to better performance of the quantized model, even on the same data. Motivated by this finding, we analyze how changes in input activation distributions affect reconstruction loss. This analysis leads to DRQ, which refines quantized weights by minimizing the worst-case reconstruction loss over a constrained set of input activation distributions. DRQ updates the integer codes representing quantized weights while preserving the quantization format and inference operators, so these improvements add no inference overhead. Experiments show that DRQ improves models quantized by six representative PTQ methods, with gains across both dense and mixture-of-experts (MoE) architectures.

%% file: tex/main/03_preliminaries.tex
\section{Preliminaries}
\label{sec:preliminaries}

\subsection{Weight-only quantization on a fixed grid}

Quantization approximates high-precision values using a finite set of values that can be stored with fewer bits. Weight-only post-training quantization applies this to pretrained weights while retaining the original activation precision. Consider a linear layer with weights $\bW\in\R^{m\times d}$, where $d$ and $m$ are the input and output dimensions. A $b$-bit uniform quantizer represents each weight $w$ by an integer code $q\in\mathcal Z_b=\{0,\ldots,2^b-1\}$. With round-to-nearest (RTN), we compute the code $q$ and its corresponding quantized weight $\hat w$ as
\begin{equation*}
q=\operatorname{clip}\!\left(\operatorname{round}(w/s)+z,\,0,\,2^b-1\right),
\qquad \hat w=s(q-z).
\end{equation*}

Here $s>0$ is the scale, which sets the spacing between adjacent quantized values, and the zero point $z$ is the integer code corresponding to zero. The rounding operation selects the nearest integer, and clipping keeps the code between $0$ and $2^b-1$. Following common practice in most PTQ work, we use group-wise quantization \citep{dettmers2023case4bit,liang2026paroquant}, which divides each row of $\bW$ into groups of $g$ consecutive weights along the input dimension. Each group shares a scale $s$ and a zero point $z$. Large weights in one group therefore do not force other groups to use a larger quantization step, which helps represent smaller weights more accurately. With $s$ and $z$ fixed, $\hat w=s(q-z)$ shows that each quantized weight is determined by its integer code $q$. DRQ therefore refines the weights by updating only their integer codes.

\subsection{Reconstruction on calibration data}

Calibration uses sample text to collect the input activations reaching each layer. Let the input activation matrix $\bX=[x_1,\ldots,x_n]\in\R^{d\times n}$ contain these vectors, with one token per column. The full-precision and quantized outputs are $\bW\bX$ and $\What\bX$. Their squared difference, summed over output channels and averaged over tokens, is the reconstruction loss
\begin{equation}
\cL(\What)
=\frac{1}{n}\norm{\What\bX-\bW\bX}_{\rm F}^{2}
=\Tr\!\left((\What-\bW)\bH(\What-\bW)^\top\right),
\label{eq:empirical-reconstruction}
\end{equation}
where $\bH=n^{-1}\bX\bX^\top$ is the activation second-moment matrix. The squared Frobenius norm sums the squared differences between the two output matrices, and $\Tr$ sums matrix diagonal entries. Most PTQ methods such as GPTQ, AWQ, and OmniQuant \citep{frantar2023gptq,lin2024awq,omniquant} minimize reconstruction loss in different ways to obtain quantized weights $\What$.

%% file: tex/main/05_marq.tex
\section{DRQ: Distributionally Robust Quantization}
\label{sec:robust-risk}
\label{sec:method}

\begin{figure*}[t]
  \centering
  \begin{minipage}{0.92\textwidth}
  \centering
  \setlength{\parskip}{0pt}
  \setlength{\abovecaptionskip}{0pt}
  \includegraphics[width=\linewidth]{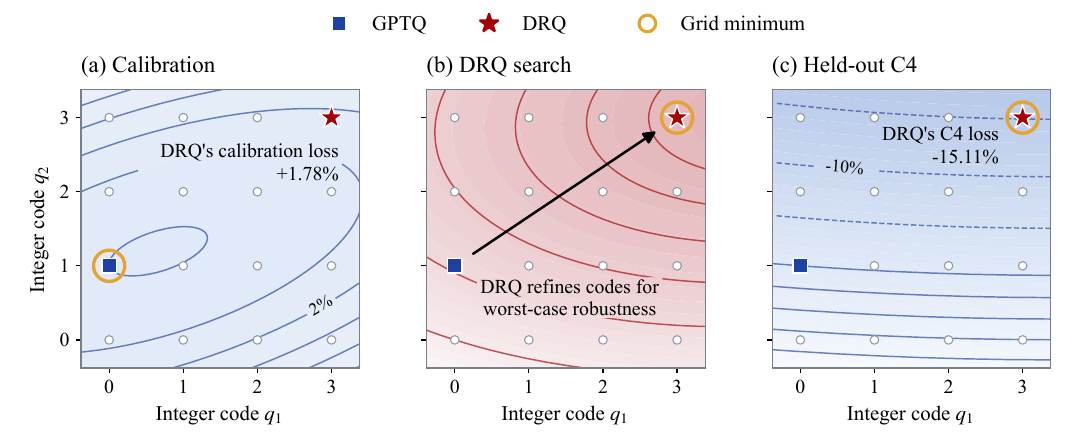}
  \caption{An illustrative example of integer-code refinement with DRQ on Llama-3.2-1B-Instruct quantized with 2-bit GPTQ. We select two integer codes, $q_1$ and $q_2$, from the quantized weights and plot the loss contours over different pairs of code values. The percentages indicate changes in reconstruction loss relative to the initial quantized weights. Panel (a) shows reconstruction loss on the WikiText-2 (WT2) calibration data, panel (b) illustrates how DRQ refines the two integer codes, and panel (c) shows reconstruction loss on C4. Although DRQ increases reconstruction loss on the calibration data by 1.78\%, it reduces reconstruction loss on C4 by 15.11\%, both relative to the initial quantized weights.}
  \label{fig:code-selection-geometry}
  \end{minipage}
\end{figure*}

In this section, we first analyze how changes in the distribution of input activations affect reconstruction loss (\cref{sec:ranking-stability}). Based on this analysis, we formulate a worst-case reconstruction objective and derive an equivalent form for optimizing quantized weights (\cref{sec:marq-objective}). Finally, we describe how DRQ optimizes this objective by updating the integer codes of weights produced by an existing PTQ method, while keeping the quantization parameters fixed (\cref{sec:marq-refinement}).

\begingroup

\subsection{Reconstruction under changing input distributions}
\label{sec:ranking-stability}

New data can produce input activations that differ from those observed during calibration. Given the quantized weights $\What$, \cref{eq:empirical-reconstruction} shows that the effect of input activations on reconstruction loss is determined by the second-moment matrix $\bH$. We can therefore study how new data affect this loss by examining changes in that matrix. To describe these changes, we introduce a perturbation matrix $\boldsymbol\Delta$ and write the second-moment matrix of the new input activations as $\bH+\boldsymbol\Delta$. Here, $\bH$ is computed from the calibration data, and $\boldsymbol\Delta$ is the difference in second-moment matrices between the new and the calibration data.

To examine how changes in the input distribution can increase reconstruction loss, we restrict $\boldsymbol\Delta$ to be positive semidefinite, written $\boldsymbol\Delta\succeq0$. Under this constraint, the expected reconstruction loss of any fixed quantized weights can increase or remain unchanged. One example is adding noise of the form $\varepsilon u$ to an input activation vector $x$, where $\varepsilon$ has mean zero and is independent of $x$. We introduce an \emph{input direction} $u\in\R^d$ to describe how the input channels change and measure the resulting increase in reconstruction loss. In the noise example above, its entries specify the sign and relative size of the noise added to each channel. For two input channels, $u=(1,1)^\top$ means adding $\varepsilon$ to both channels, so they increase or decrease together. With $u=(1,-1)^\top$, we add $\varepsilon$ to the first channel and subtract it from the second, so one increases while the other decreases. In this example, the second-moment matrix changes by $\boldsymbol\Delta=\mathbb{E}[\varepsilon^2]uu^\top\succeq0$. For fixed quantized weights, this change increases the expected reconstruction loss by $\mathbb{E}[\varepsilon^2]\norm{(\What-\bW)u}_2^2$. Comparing input directions under the same distribution-change budget therefore lets us identify the changes that increase this loss the most.

The loss increase in the example above depends on both the input direction and the size of the added noise. Without a limit on the noise, reconstruction loss can grow arbitrarily large. We therefore impose $\Tr(\bC\boldsymbol\Delta)\le\alpha$, where $\alpha\ge0$ is the distribution-change budget. The positive-definite matrix $\bC$ determines how much change is permitted along each input direction. For a change $\boldsymbol\Delta=tuu^\top$ with $t\ge0$, the constraint becomes $t\,u^\top\bC u\le\alpha$. Thus, a larger $u^\top\bC u$ permits a smaller $t$ under the same budget. To compare all directions at the same budget, we rescale each $u$ to satisfy $u^\top\bC u=1$. Then $\boldsymbol\Delta=\alpha uu^\top$ uses the full budget $\alpha$.

Panels~(a) and~(c) of \cref{fig:code-selection-geometry} illustrate how weights with lower reconstruction loss on calibration data can have higher loss on new inputs. Consider two quantized weight matrices with $\cL(\What_1)<\cL(\What_2)$. Input changes can increase both losses, but increase the loss of $\What_1$ more. For its loss to remain no greater than that of $\What_2$, the difference between these increases must not exceed the original loss difference $\cL(\What_2)-\cL(\What_1)$. The following theorem gives the exact condition.

\begin{theorem}[Reconstruction loss under distribution changes]
\label{thm:ordering-certificate}
Let $\What_1$ and $\What_2$ satisfy $\cL(\What_1)<\cL(\What_2)$. Consider input activation distributions with second-moment matrix $\bH+\boldsymbol\Delta$, where $\boldsymbol\Delta\succeq0$ and $\Tr(\bC\boldsymbol\Delta)\le\alpha$. The expected reconstruction loss of $\What_1$ is no greater than that of $\What_2$ for every such distribution if and only if every direction $u$ normalized by $u^\top\bC u=1$ satisfies
\begin{equation}
\alpha\!\left(\norm{(\What_1-\bW)u}_2^2-\norm{(\What_2-\bW)u}_2^2\right)\le\cL(\What_2)-\cL(\What_1).
\label{eq:directional-ranking-stability}
\end{equation}
If the inequality fails for a direction $u$ with $u^\top\bC u=1$, the change $\boldsymbol\Delta=\alpha uu^\top$ satisfies both constraints and makes $\What_1$ have higher expected reconstruction loss than $\What_2$.
\end{theorem}
See Appendix~\ref{app:ordering-certificate} for the proof. \Cref{thm:ordering-certificate} motivates a refinement objective that accounts for both reconstruction loss on calibration data and its largest possible increase under the distribution constraints above. DRQ, introduced next, follows this approach. In the schematic in \cref{fig:code-selection-geometry}, panel~(b) illustrates how DRQ refines codes to accept slightly higher reconstruction loss on calibration data in panel~(a) and achieve lower loss on held-out inputs in panel~(c).

\Needspace{4\baselineskip}
\subsection{Robust reconstruction objective}
\label{sec:marq-objective}

The preceding analysis compares two weight matrices under the same distribution changes. For refinement, we need an objective for evaluating each weight matrix. The proposed method, Distributionally Robust Quantization (DRQ), therefore considers the largest reconstruction loss that each matrix can have under the stated constraints and seeks weights that make this loss as small as possible. Using the same distribution-change budget $\alpha$, we minimize
\begin{equation}
\min_{\What\in\cQ}\; \max_{\boldsymbol\Delta\succeq0,\,\Tr(\bC\boldsymbol\Delta)\le\alpha}
\Tr\!\left((\What-\bW)(\bH+\boldsymbol\Delta)(\What-\bW)^\top\right) ~,
\label{eq:marq-minmax}
\end{equation}

where $\cL_{\rm DRQ}(\What):= \max_{\boldsymbol\Delta\succeq0,\,\Tr(\bC\boldsymbol\Delta)\le\alpha}
\Tr\!((\What-\bW)(\bH+\boldsymbol\Delta)(\What-\bW)^\top)$. The outer minimization selects quantized weights on the existing quantization grid that minimize this loss. To evaluate $\cL_{\rm DRQ}(\What)$, the following theorem reduces the inner maximization to finding the input direction $u$ along which reconstruction loss increases the most under the same distribution-change budget. The resulting expression is reconstruction loss on calibration data plus this worst-case increase.

\begin{theorem}[Worst-case reconstruction loss]
\label{thm:exact-risk}
For any $\What\in\R^{m\times d}$, $\bC\succ0$, and $\alpha\ge0$, the worst-case loss in \cref{eq:marq-minmax} satisfies
\begin{equation}
\cL_{\rm DRQ}(\What)
=\cL(\What)+\alpha\cS_{\bC}(\What),
\label{eq:marq-loss}
\end{equation}
where the \emph{worst-case directional error} is
\begin{equation}
\cS_{\bC}(\What)
=\max_{u^\top\bC u=1}\norm{(\What-\bW)u}_2^2.
\label{eq:coverage-risk}
\end{equation}
\end{theorem}
For an input vector $u$, $(\What-\bW)u=\What u-\bW u$ is the difference between the quantized and full-precision layer outputs. Among directions rescaled to satisfy $u^\top\bC u=1$, the \emph{worst-case input direction} gives the largest squared output difference, $\cS_{\bC}(\What)$. A distribution change $\boldsymbol\Delta=\alpha uu^\top$ along this direction uses the full budget and increases reconstruction loss by $\alpha\cS_{\bC}(\What)$. This is the largest increase among all distribution changes satisfying the constraints.

In \cref{eq:marq-minmax}, $\alpha$ sets the budget for changes in the input activation distribution. In the equivalent objective in \cref{eq:marq-loss}, it weights the worst-case directional error relative to reconstruction loss on calibration data. DRQ can accept slightly higher reconstruction loss on calibration data when this reduces the largest possible increase enough to lower the total worst-case loss. Appendix~\ref{app:exact-risk} gives the proof and a distribution change that attains this loss.

We introduce $\bC$ to account for input activations that remain small during calibration. When an input channel stays near zero, even large differences between the quantized and full-precision weights in the corresponding column can have little effect on the layer outputs. We therefore consider larger possible increases in these activations when evaluating the worst-case loss. To apply the same idea to combinations of channels, we set $\bC=(1-\beta)\bH+\beta\bI$, with $0<\beta\le1$. For directions with the same Euclidean norm, a smaller $u^\top\bH u$ gives a smaller $u^\top\bC u$ when $\beta<1$, permitting a larger increase under the same budget. At $\beta=1$, all directions with the same Euclidean norm are treated equally. The identity term keeps $\bC$ positive definite, so changes remain bounded in every direction.

\subsection{Refinement on a fixed quantization grid}
\label{sec:marq-refinement}

DRQ processes one Transformer block at a time and refines each linear layer after the base PTQ method quantizes its weights. It changes only integer codes, keeping the bit width, grouping, scales, and zero points fixed.\par

\begingroup
\vspace{-\parskip}
\restylefloat{algorithm}
\setlength{\intextsep}{-3.25pt}
\setlength{\columnsep}{12pt}
\begin{wrapalgorithm}{r}{0.66\textwidth}
\vspace{-0.45\baselineskip}
\setlength{\parskip}{0pt}
\caption{DRQ refinement of one linear layer}
\label{alg:drq-overview}
\begingroup
\small
\algrenewcommand\algorithmicindent{1em}
\vspace{4pt}
\begin{algorithmic}[1]
\Require \begin{tabular}[t]{@{}l@{}}
Full-precision weights $\bW$, initial quantized weights $\What$,\\
calibration activations $\bX$, $\alpha,\beta$, maximum rounds $K$
\end{tabular}
\State $\What_{\rm DRQ}\gets\What$,\quad $\bH\gets\bX\bX^\top/n,\quad \bC\gets(1-\beta)\bH+\beta\bI$
\For{$r=1,\ldots,K$}
  \State $u\gets\argmax_{u^\top\bC u=1}\norm{(\What_{\rm DRQ}-\bW)u}_2^2$
  \State Rank integer-code updates by their reduction in
  \Statex \hspace{\algorithmicindent}$\cL(\What_{\rm DRQ})+\alpha\norm{(\What_{\rm DRQ}-\bW)u}_2^2$.
  \State Apply the first $k$ ranked code updates to $\What_{\rm DRQ}$ for several values of $k$ to generate candidates $\mathcal P$.
  \State $\begin{aligned}[t]
    \cL_{\rm DRQ}(\What_k)&\gets\cL(\What_k)\\
      &\quad+\alpha\max_{u^\top\bC u=1}\norm{(\What_k-\bW)u}_2^2,\quad\forall\What_k\in\mathcal P
  \end{aligned}$
  \State $\mathcal P\gets\{\What_k\in\mathcal P:\ \cL_{\rm DRQ}(\What_k)<\cL_{\rm DRQ}(\What_{\rm DRQ})\}$
  \State \textbf{if} $\mathcal P=\varnothing$ \textbf{then break}
  \State $\What_{\rm DRQ}\gets\argmin_{\What_k\in\mathcal P}\cL_{\rm DRQ}(\What_k)$
\EndFor
\State \Return $\What_{\rm DRQ}$
\end{algorithmic}
\endgroup
\vspace{1pt}
\end{wrapalgorithm}

\noindent Algorithm~\ref{alg:drq-overview} summarizes how we optimize \cref{eq:marq-loss}, using numerical estimates of the worst-case directions and losses. The set $\mathcal P$ contains proposed weight matrices. To preserve the fit to calibration data, reconstruction loss may increase only by a preset percentage relative to the initial weights, on both the search subset and separate checking subsets. All refinement is offline and leaves inference operators unchanged, so DRQ adds no inference overhead. Appendix~\ref{app:algorithm-details} gives the search procedure and parameter settings, and Appendix~\ref{app:moe-experts} covers sparsely activated experts.\par

\par
\WFclear
\endgroup
\endgroup

%% file: tex/main/06_experiments.tex
\section{Experimental Evaluation}
\label{sec:experiments}
\suppressfloats[t]

\subsection{Experimental setup}

\paragraph{Models and tasks.}
We apply DRQ to both dense and mixture-of-experts (MoE) models, including Llama-3.2-Instruct (1B, 3B), Llama-3.1-Instruct (8B, 70B) \citep{llama3}, and Qwen3-30B-A3B \citep{qwen3}. We evaluate model quality after quantization using two types of metrics: (1) perplexity on WikiText-2 (WT2) \citep{wikitext} and C4 \citep{raffel2020c4}; (2) zero-shot accuracy on seven downstream tasks, including PIQA \citep{bisk2020piqa}, HellaSwag \citep{zellers2019hellaswag}, WinoGrande \citep{sakaguchi2019winogrande}, ARC-Easy and ARC-Challenge \citep{clark2018arc}, BoolQ \citep{clark2019boolq}, and MMLU \citep{hendrycks2021mmlu}.

\paragraph{Implementation.}
We evaluate DRQ at W4A16, W3A16, and W2A16 to study how well it preserves model quality across different weight bit widths. All three settings use group-wise weight quantization with a group size of 128 \citep{frantar2023gptq,lin2024awq}. Unless otherwise specified, we calibrate with 128 WikiText-2 training sequences of 2,048 tokens and select DRQ hyperparameters using 64 validation samples from Pile \citep{gao2020pile}. Experiments run on NVIDIA A800 80GB GPUs unless otherwise specified. Further implementation details are provided in Appendix~\ref{app:reproducibility}.

\paragraph{Baselines.}
We use GPTQ \citep{frantar2023gptq} and AWQ \citep{lin2024awq}, two widely used weight-only PTQ methods, as the main baselines for most experiments. To demonstrate that DRQ generalizes across quantization methods, we also include round-to-nearest (RTN), OmniQuant \citep{omniquant}, GPTAQ \citep{li2025gptqv2}, and ParoQuant \citep{liang2026paroquant}.

\begin{table*}[!htbp]
\centering
\input{tex/tables/table1_dense.tex}
\end{table*}

\subsection{Dense-model results across sizes and bit widths}

\Cref{tab:mock-dense} shows that DRQ improves average task accuracy across both GPTQ and AWQ at every evaluated model size and bit width, while reducing perplexity in most cases. The gains are particularly pronounced when aggressive quantization causes substantial quality loss. For Llama-3.1-70B with GPTQ W2, DRQ lowers C4 PPL from 31.77 to 27.11 and raises accuracy from 48.23\% to 54.56\%. Similarly, for Llama-3.2-3B with GPTQ W3, C4 PPL falls from 30.34 to 22.11 and accuracy increases from 55.83\% to 60.24\%. These results show that refinement can recover part of the quality lost at low precision, although severe AWQ-W2 degradation persists. Its benefits also extend to stronger initial quantized models. At W4, where the baselines already retain much of their full-precision accuracy, DRQ improves accuracy in all eight comparisons. Appendix~\ref{app:dense-results} provides the detailed results.

\subsection{Extension to mixture-of-experts models}

\Cref{tab:mock-moe} extends the evaluation to Qwen3-30B-A3B at W3A16. DRQ lowers perplexity on both WT2 and C4 and improves accuracy on five of the seven downstream tasks. Together with the dense-model results, these improvements demonstrate DRQ's applicability across dense and MoE architectures. Further details on MoE quantization and expert refinement are provided in Appendix~\ref{app:moe-experts}.

\begingroup
\setlength{\intextsep}{4pt}
\begin{table}[H]
\centering
\input{tex/tables/table2_moe.tex}
\end{table}
\endgroup

\clearpage
\subsection{Versatility across six PTQ methods}
DRQ improves models quantized by all six PTQ methods, reducing both WT2 and C4 perplexity and increasing average accuracy on Llama-3.2-3B W3 (\cref{fig:cross-host-radar}). ParoQuant achieves the best results before refinement and retains its lead on all three metrics after applying DRQ, showing that even a strong initial quantized model can benefit from refinement. Detailed task-level results and computation costs are provided in Appendices~\ref{app:cross-host-results} and~\ref{app:cross-host-cost}.

\begin{figure}[!htbp]
\centering
\setlength{\parskip}{0pt}
\setlength{\abovecaptionskip}{-1pt}
\includegraphics[trim=0 12pt 0 0,clip,width=0.94\textwidth]{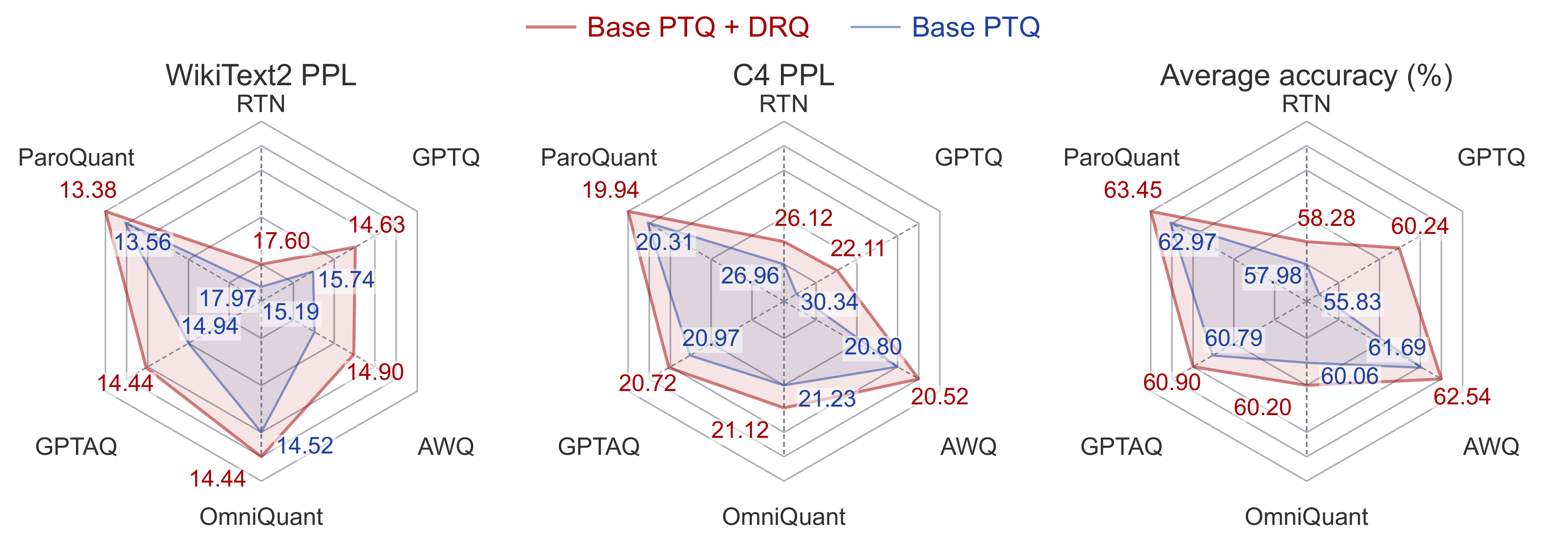}
\caption{Six-method comparison on Llama-3.2-3B W3.}
\label{fig:cross-host-radar}
\end{figure}

\FloatBarrier
\subsection{Calibration size and the choice of objective}
\label{sec:calibration-study}

We vary calibration size from 16 to 256 sequences of 2,048 tokens from WT2 or C4 for Llama-3.2-3B-Instruct with GPTQ W4. Reconstruction-only refinement (rec-only) uses the same code search as DRQ with $\alpha=0$.

\vspace{-2pt}
\begin{figure}[H]
\centering
\setlength{\parskip}{0pt}
\setlength{\abovecaptionskip}{-4pt}
\includegraphics[trim=0 6pt 0 0,clip,width=0.80\linewidth]{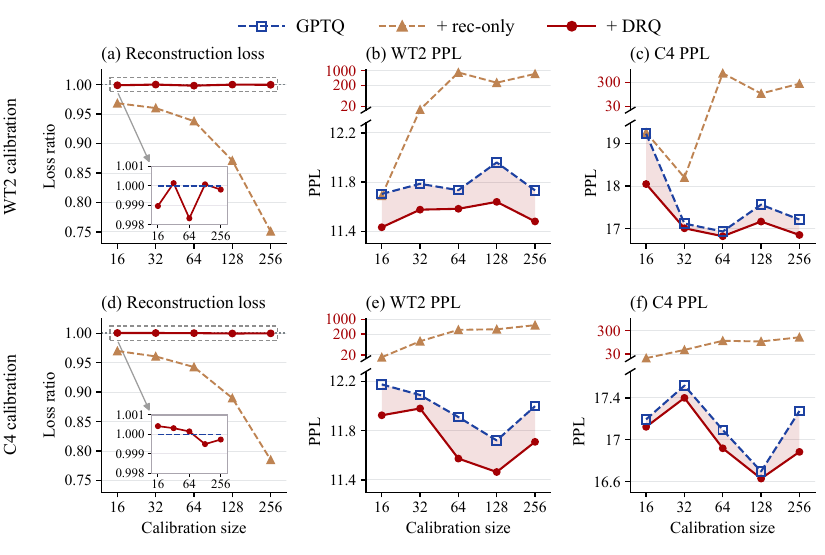}
\caption{Reconstruction loss and PPL across calibration sizes for Llama-3.2-3B-Instruct with GPTQ W4. Loss ratios divide the reconstruction loss summed across layers after refinement by that before refinement, using each run's own calibration activations.}
\label{fig:calibration-budget}
\end{figure}

The first column of \cref{fig:calibration-budget} shows larger relative loss reductions from rec-only as calibration size grows. This further reduction is possible because GPTQ does not fully minimize reconstruction loss over the larger calibration sets. Yet even as rec-only reduces the loss further relative to GPTQ, model performance worsens, with PPL increasing by orders of magnitude. DRQ instead keeps reconstruction loss close to GPTQ's while achieving lower PPL on both datasets. These results show that more calibration data need not resolve this mismatch, as also established theoretically in \cref{thm:population-persistence}. \Cref{fig:parameter-sensitivity} in Appendix~\ref{app:parameter-sensitivity} examines $\alpha$ and $\beta$ at W3.

\FloatBarrier
\Needspace{16\baselineskip}
\noindent
\begin{minipage}[b]{0.625\textwidth}
\vspace{0pt}
\subsection{Efficiency of DRQ}
\label{sec:efficiency}
\Cref{fig:quantization-cost} shows the offline cost of DRQ on Llama-3.1-8B at W3. We time the loop that quantizes successive Transformer blocks, including forward passes on calibration data. Across 20 configurations, refinement takes almost the same time with GPTQ and AWQ, averaging 6.01 and 5.96 minutes, respectively. It is because DRQ applies the same code refinement within the existing quantization grid in both cases. Further efficiency analysis appears in Appendix~\ref{app:efficiency}.
\end{minipage}\hspace{0.025\textwidth}%
\begin{minipage}[b]{0.33\textwidth}
\vspace{0pt}
\centering
\setlength{\parskip}{0pt}
\setlength{\abovecaptionskip}{1pt}
\setlength{\belowcaptionskip}{0pt}
\includegraphics[width=\linewidth]{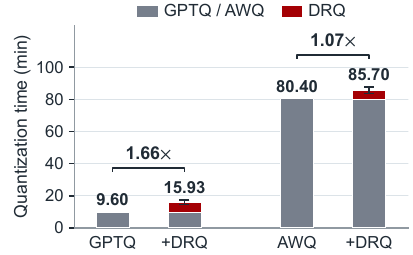}\par
\small
\makeatletter
\long\def\@makecaption#1#2{%
  \vskip\abovecaptionskip
  {\centering #1: #2\par}%
  \vskip\belowcaptionskip}
\makeatother
\captionof{figure}{Quantization time for \mbox{Llama-3.1-8B W3.}}
\label{fig:quantization-cost}
\end{minipage}
\par

%% file: tex/tables/table1_dense.tex
\begingroup
\definecolor{MockBlue}{HTML}{176B91}
\definecolor{MockInk}{HTML}{202B35}
\definecolor{MockMuted}{HTML}{5E6873}

\caption{Perplexity and average task accuracy (\%) of dense Llama models across model sizes and weight bit widths. FP denotes the unquantized full-precision model.}
\label{tab:mock-dense}
\fontsize{8.25}{10.25}\selectfont
\setlength{\tabcolsep}{2.7pt}
\renewcommand{\arraystretch}{1.00}
\noindent
\begin{tabular*}{\linewidth}{@{\extracolsep{\fill}}cl*{12}{c}@{}}
\toprule
\multicolumn{2}{l}{Quantization} & \multicolumn{3}{c}{\textbf{Llama-3.2-1B}} & \multicolumn{3}{c}{\textbf{Llama-3.2-3B}} & \multicolumn{3}{c}{\textbf{Llama-3.1-8B}} & \multicolumn{3}{c}{\textbf{Llama-3.1-70B}} \\
\cmidrule(lr){3-5}\cmidrule(lr){6-8}\cmidrule(lr){9-11}\cmidrule(lr){12-14}
Bits & Method & WT2 & C4 & ACC & WT2 & C4 & ACC & WT2 & C4 & ACC & WT2 & C4 & ACC \\
\midrule
FP & Unquantized & 13.16 & 20.82 & 58.68 & 11.05 & 16.19 & 66.48 & 7.22 & 11.24 & 74.40 & 3.78 & 8.24 & 80.64 \\
\midrule
\multirow{4}{*}{W2} & GPTQ & 1719.88 & 1.03e4 & 34.75 & 141.54 & 645.35 & 34.58 & 139.90 & 681.48 & 35.49 & 13.92 & 31.77 & 48.23 \\
 & \textcolor{MockBlue}{\quad + DRQ} & \textbf{1402.87} & \textbf{7233.14} & \textbf{35.93} & \textbf{130.02} & \textbf{585.86} & \textbf{35.03} & \textbf{98.67} & \textbf{390.36} & \textbf{36.64} & \textbf{12.48} & \textbf{27.11} & \textbf{54.56} \\
\addlinespace[1pt]
 & AWQ & 1.08e5 & 1.22e5 & 35.39 & 2.42e5 & 5.57e5 & 35.38 & 1.63e6 & 1.84e6 & 37.80 & 1.42e6 & 1.25e6 & 37.99 \\
 & \textcolor{MockBlue}{\quad + DRQ} & \textbf{1.31e4} & \textbf{1.80e4} & \textbf{36.51} & \textbf{1.85e4} & \textbf{1.03e5} & \textbf{38.38} & \textbf{7.76e5} & \textbf{5.66e5} & \textbf{38.01} & \textbf{1.31e6} & \textbf{1.17e6} & \textbf{38.01} \\
\midrule
\multirow{4}{*}{W3} & GPTQ & 18.65 & 28.84 & 50.50 & 15.74 & 30.34 & 55.83 & 8.60 & 15.39 & 68.46 & 5.41 & 10.01 & 78.67 \\
 & \textcolor{MockBlue}{\quad + DRQ} & \textbf{17.64} & \textbf{28.42} & \textbf{51.68} & \textbf{14.63} & \textbf{22.11} & \textbf{60.24} & 8.61 & \textbf{14.91} & \textbf{69.48} & \textbf{5.39} & \textbf{9.92} & \textbf{79.03} \\
\addlinespace[1pt]
 & AWQ & 21.12 & 32.73 & 52.65 & 15.19 & 20.80 & 61.69 & 9.01 & 14.43 & 69.90 & 5.38 & 9.77 & 79.02 \\
 & \textcolor{MockBlue}{\quad + DRQ} & \textbf{20.89} & \textbf{31.82} & \textbf{53.08} & \textbf{14.90} & \textbf{20.52} & \textbf{62.54} & \textbf{9.00} & \textbf{14.43} & \textbf{70.10} & \textbf{5.37} & 9.77 & \textbf{79.23} \\
\midrule
\multirow{4}{*}{W4} & GPTQ & 14.27 & 22.24 & 57.47 & 11.96 & 17.57 & 65.26 & 7.48 & 12.07 & 73.11 & 4.17 & 8.59 & 80.22 \\
 & \textcolor{MockBlue}{\quad + DRQ} & \textbf{14.21} & \textbf{22.12} & \textbf{58.13} & \textbf{11.64} & \textbf{17.17} & \textbf{65.96} & \textbf{7.46} & \textbf{12.02} & \textbf{74.00} & \textbf{4.14} & \textbf{8.57} & \textbf{80.55} \\
\addlinespace[1pt]
 & AWQ & 14.55 & 22.63 & 57.27 & 11.56 & 16.89 & 66.18 & 7.55 & 12.03 & 73.49 & 4.18 & 8.62 & 80.38 \\
 & \textcolor{MockBlue}{\quad + DRQ} & \textbf{14.41} & \textbf{22.54} & \textbf{57.43} & \textbf{11.52} & \textbf{16.81} & \textbf{66.70} & \textbf{7.55} & \textbf{12.02} & \textbf{73.62} & 4.19 & \textbf{8.61} & \textbf{80.52} \\
\bottomrule
\end{tabular*}
\endgroup

%% file: tex/tables/table2_moe.tex
\begingroup
\definecolor{MockBlue}{HTML}{176B91}
\definecolor{MockInk}{HTML}{202B35}

\caption{Perplexity and seven-task accuracy (\%) for Qwen3-30B-A3B W3A16.}
\label{tab:mock-moe}
\fontsize{8.25}{10.25}\selectfont
\setlength{\tabcolsep}{1.5pt}
\renewcommand{\arraystretch}{1.10}
\noindent
\begin{tabular*}{\linewidth}{@{\extracolsep{\fill}}cl*{10}{c}@{}}
\toprule
\multicolumn{2}{l}{Quantization} & \multicolumn{2}{c}{Perplexity} & \multicolumn{8}{c}{Accuracy (\%)} \\
\cmidrule(lr){3-4}\cmidrule(lr){5-12}
Bits & Method & WT2 & C4 & PIQA & HellaSwag & MMLU & BoolQ & ARC-C & ARC-E & WinoGrande & Mean \\
\midrule
\multirow{2}{*}{W3} & GPTQ & 9.29 & 15.01 & 78.24 & 75.85 & 72.16 & 86.15 & 49.57 & 73.99 & 68.90 & 72.12 \\
 & \textcolor{MockBlue}{\quad + DRQ} & \textbf{9.20} & \textbf{14.90} & \textbf{78.78} & \textbf{76.00} & \textbf{73.07} & 85.87 & \textbf{52.99} & \textbf{76.05} & 67.56 & \textbf{72.90} \\
\bottomrule
\end{tabular*}
\endgroup

%% file: tex/main/02_related_work.tex
\section{Related Work}
\label{sec:related-work}

\paragraph{Post-training quantization of LLMs.} PTQ methods improve low-bit approximation by reducing the effect of outliers and optimizing quantized weights. Scaling and clipping control channel magnitudes and quantization ranges \citep{lin2024awq,xiao2023smoothquant,omniquant}, while rotations and affine transformations redistribute values across channels \citep{Ashkboos2024quarot,liuspinquant,liang2026paroquant,sun2025flatquant, zhao2026saliency}. QuIP and QTIP combine randomized transformations with adaptive rounding and trellis coding, respectively \citep{Chee2023quip,tseng2024qtip}. For weight optimization, GPTQ uses second-order error compensation during rounding \citep{frantar2023gptq}. Complementary work changes the reconstruction objective to better preserve model performance. GPTAQ accounts for errors from preceding quantized layers when matching full-precision layer outputs, and CoreQ adjusts the strength of this correction across layers \citep{li2025gptqv2,cha2026coreq}. GuidedQuant uses gradients of the model's final loss to weight reconstruction loss \citep{KimEtAl2025GuidedQuant}, while LFQ matches the model's token probabilities to those of the full-precision model when quantizing the final Transformer block \citep{lee2026lfq}.

\paragraph{Distributionally robust optimization.} Distributionally robust optimization (DRO) minimizes the worst expected loss over a set of plausible distributions. Moment constraints, divergence bounds, and transport distances provide different ways to define this uncertainty set \citep{delage2010moment,esfahani2018wasserstein,levy2020largescale}. Group DRO focuses on shifts among predefined groups, while other approaches learn or constrain how an adversary reweights individual samples \citep{sagawa2020groupdro,michel2022likelihood,setlur2023bitrate}. The choice of set determines which distribution changes the objective protects against. Allowing too much freedom can lead to overly conservative solutions or amplify the influence of outliers. Geometry-aware approaches use data structure to constrain these changes or calibrate robust risks \citep{liu2022datageometry,liu2024geometrycalibrated}, while outlier-resistant formulations address contamination in the data \citep{zhai2021doro}. Connections to regularization further explain how uncertainty sets induce penalties on the learning objective \citep{husain2020ipm,cranko2021lipschitz}. Research on these formulations also establishes population-risk guarantees for Wasserstein DRO \citep{azizian2023exact,le2025universal} and develops scalable solvers \citep{haddadpour2022composite,mehta2024drago,zhang2025revisitingdro}.

%% file: tex/main/07_conclusion.tex
\section{Conclusion}

The weights that minimize the reconstruction loss on calibration data need not minimize reconstruction loss on new inputs or yield the best model performance. We establish this gap between calibration fit and performance at deployment and show why exact optimization and more calibration data need not close it both theoretically and empirically. Our analysis characterizes how input changes reverse reconstruction rankings and motivates a distributionally robust objective that accounts for those changes. The proposed method, DRQ, combines reconstruction loss on calibration data with a worst-case directional error term to refine existing integer codes of quantized weights. This improves the selection criterion while preserving the quantization grid, scales, and zero points. Experiments across multiple PTQ methods and dense and mixture-of-experts models demonstrate better perplexity and average task accuracy. By retaining the original quantization format and inference operators, DRQ turns robust reconstruction into better downstream performance without adding extra computational overhead during inference.

%% file: tex/appendix/00_appendix.tex
\input{tex/appendix/A_notation}
\input{tex/appendix/B_proofs}
\input{tex/appendix/C_algorithm_details}
\input{tex/appendix/D_00_experimental_results}

\input{tex/appendix/E_00_efficiency}

%% file: tex/appendix/A_notation.tex
\section{Notation}
\label{app:notation}

\begin{table}[H]
\caption{Summary of notation.}
\label{tab:notation}
\centering
\small
\begin{tabular}{@{}p{0.22\linewidth}p{0.73\linewidth}@{}}
\toprule
Symbol & Meaning \\
\midrule
$\bW,\What(\mathbf q)$ & Full-precision and quantized weights in $\R^{m\times d}$, after any scaling or rotation by the base PTQ method \\
$\mathbf q,\cQ$ & Integer codes representing quantized weights and the grid defined by fixed bit width, grouping, scales, and zero points \\
$\What_0$ & \leavevmode \settingrevision{Initial quantized weights produced by the base PTQ method, used to limit increases in reconstruction loss} \\
$\bX$ & \leavevmode \settingrevision{Calibration input activation matrix in $\R^{d\times n}$, with one token per column} \\
$\bH$ & Activation second-moment matrix $\bX\bX^\top/n$ \\
$\bX_{\rm d},\bX_{\rm c}^{(h)}$ & \leavevmode \settingrevision{Input activations in the search and checking subsets, with $n_{\rm d}$ and $n_{\rm c}^{(h)}$ columns. DRQ computes $\bH$ and $\cL$ from the search subset} \\
$\cL(\What)$ & \leavevmode Layer-output reconstruction loss $\norm{\What\bX-\bW\bX}_F^2/n$ \\
$\cJ(\What)$ & Fixed evaluation loss used in \cref{thm:population-persistence}, such as model token NLL after replacing one layer's weights \\
\addlinespace[3pt]
\leavevmode \settingrevision{$e_{12}(\bC)$} & \leavevmode \settingrevision{Largest increase in the reconstruction loss of $\What_1$ minus that of $\What_2$ under $\boldsymbol\Delta\succeq0$ and $\Tr(\bC\boldsymbol\Delta)\le1$} \\
$\boldsymbol\Delta$ & \leavevmode \settingrevision{Change in the input activation second-moment matrix from $\bH$ to $\bH+\boldsymbol\Delta$, with $\boldsymbol\Delta\succeq0$ and $\Tr(\bC\boldsymbol\Delta)\le\alpha$} \\
$\bC$ & \leavevmode \settingrevision{Positive-definite matrix that weights input changes in the budget $\Tr(\bC\boldsymbol\Delta)\le\alpha$. DRQ uses $(1-\beta)\bH+\beta\bI$} \\
$\cS_{\bC}(\What)$ & \leavevmode Worst-case directional error $\max_{u^\top\bC u=1}\norm{(\What-\bW)u}_2^2$ \\
$\alpha,\beta$ & \leavevmode \settingrevision{Distribution-change budget (also the weight on worst-case directional error) and the weight on the identity matrix in $\bC$} \\
$\cL_{\rm DRQ}(\What)$ & Worst-case reconstruction loss under $\boldsymbol\Delta\succeq0$ and $\Tr(\bC\boldsymbol\Delta)\le\alpha$, equal to $\cL(\What)+\alpha\cS_{\bC}(\What)$ \\
\bottomrule
\end{tabular}
\end{table}

\FloatBarrier

%% file: tex/appendix/B_proofs.tex
\section{Theoretical Results and Proofs}
\label{app:proofs}

\settingrevision{This appendix supports the analysis and objective in \cref{sec:method}. We first explain how a fixed quantization grid can make reconstruction loss on calibration data favor weights with higher loss on new inputs, and why more calibration data need not resolve this difference. We then prove the results for changes in the input activation distribution and explain the guarantees of the DRQ objective. Throughout, the full-precision weights and quantization grid remain fixed. The matrix $\bH$ is the second-moment matrix of the calibration input activations, computed from samples or from their distribution as specified below.}

\subsection{The role of the quantization grid}
\label{app:quantization-specific}

Without a quantization constraint, retaining the full-precision weights gives zero reconstruction loss for any input distribution. Quantization usually removes this option. The following result makes this distinction explicit.

\begin{proposition}[Reconstruction when full-precision weights are representable]
\label{prop:quantization-specific}
For every activation second-moment matrix $\bH\succeq0$, the unconstrained problem $\min_{\What\in\R^{m\times d}}\norm{(\What-\bW)\bH^{1/2}}_{\rm F}^{2}$ has $\bW$ as a global minimizer. This minimizer is unique when $\bH\succ0$. If $\bW\in\cQ$, it also minimizes reconstruction loss over the quantization grid for every input distribution with a finite second moment. Thus two positive-definite second-moment matrices can select different unique minimizers only if the grid excludes $\bW$.
\end{proposition}

\begin{proof}[Proof of \cref{prop:quantization-specific}]
For any $\What$ and $\bH\succeq0$, $\norm{(\What-\bW)\bH^{1/2}}_{\rm F}^{2}\ge0$. Setting $\What=\bW$ gives zero loss. If $\bH\succ0$, equality requires $\What=\bW$, so the minimizer is unique. The same argument applies when optimization is restricted to a grid containing $\bW$.
\end{proof}

Excluding $\bW$ makes different choices possible, but does not force them. The fixed-grid example in \cref{thm:finite-bit-trap} shows that the weights minimizing reconstruction loss on calibration data can have larger loss on a new input.

\subsection{Why more calibration data may not resolve the mismatch}
\label{app:population-persistence}

More calibration data can estimate reconstruction loss more accurately. This does not necessarily improve the selected weights if minimizing the expected reconstruction loss already favors weights with a higher evaluation loss. The following result separates this mismatch between objectives from error due to a limited calibration sample.

\settingrevision{Let $\cJ(\What)$ be a fixed, finite evaluation loss, with smaller values indicating better performance. It may be reconstruction loss on new input activations or the model's token negative log-likelihood (NLL) after replacing one layer's weights while keeping all other weights fixed. The evaluation rule and distribution of evaluation data remain unchanged as the number of calibration samples increases.}

\begin{theorem}[More calibration data need not improve the selected weights]
\label{thm:population-persistence}
Let $\cQ$ be a fixed finite quantization grid, and let $x_1,x_2,\ldots$ be independent samples from the same calibration input activation distribution with $\mathbb E\norm{x}_2^2<\infty$. Suppose $\What\in\cQ$ uniquely minimizes $\mathbb E\norm{(\What-\bW)x}_2^2$ over this grid, while another matrix $\What'\in\cQ$ satisfies $\cJ(\What')<\cJ(\What)$. With probability one, every exact minimizer of reconstruction loss on the first $n$ samples equals $\What$ for all sufficiently large $n$. Its evaluation loss therefore remains worse than that of $\What'$ by the fixed positive amount $\cJ(\What)-\cJ(\What')$.
\end{theorem}

\begin{proof}[Proof of \cref{thm:population-persistence}]
For each quantized weight matrix $\mathbf V\in\cQ$, write the reconstruction loss on the first $n$ calibration inputs as
\begin{equation}
\widehat\cL_n(\mathbf V)=\frac{1}{n}\sum_{k=1}^{n}
\norm{(\mathbf V-\bW)x_k}_2^2.
\end{equation}
\settingrevision{The finite-second-moment assumption ensures that this squared loss has a finite expectation.} The strong law of large numbers gives $\widehat\cL_n(\mathbf V)\to\mathbb E\norm{(\mathbf V-\bW)x}_2^2$ almost surely. Because the grid is finite, the largest absolute difference between empirical and expected loss over all matrices in $\cQ$ also converges to zero.

The unique minimizer $\What$ of expected reconstruction loss is separated from every other matrix by a positive gap
\begin{equation}
\gamma=\min_{\mathbf V\in\cQ,\,\mathbf V\ne\What}
\left\{\mathbb E\norm{(\mathbf V-\bW)x}_2^2
-\mathbb E\norm{(\What-\bW)x}_2^2\right\}>0.
\end{equation}
With probability one, all empirical losses eventually lie within $\gamma/3$ of their expectations. Every other matrix then has empirical loss at least $\gamma/3$ larger than that of $\What$. The exact empirical minimizer therefore equals $\What$ for all sufficiently large $n$.

The selected weights then retain the evaluation-loss gap
\begin{equation}
\cJ(\What)-\cJ(\What')>0,
\end{equation}
which proves the claim.
\end{proof}

The theorem assumes that the expected reconstruction loss and evaluation loss favor different weights. It shows that more samples from the same distribution cannot remove this existing mismatch simply by estimating reconstruction loss more accurately. It does not require evaluation on a different dataset, since $\cJ$ can also measure model performance on the calibration distribution. \Cref{thm:finite-bit-trap} gives a concrete instance with reconstruction loss on a new input as $\cJ$.

\subsection{When input changes reverse the reconstruction loss comparison}
\label{app:ordering-certificate}

We now compare two quantized weight matrices with $\cL(\What_1)<\cL(\What_2)$. As in \cref{sec:ranking-stability}, new input activations have second-moment matrix $\bH+\boldsymbol\Delta$, where $\boldsymbol\Delta\succeq0$ and $\Tr(\bC\boldsymbol\Delta)\le\alpha$, with fixed $\bC\succ0$. The proof of \cref{thm:ordering-certificate} finds the largest possible increase in the loss of $\What_1$ relative to that of $\What_2$ under these constraints.

\settingrevision{For the proof, we use two auxiliary quantities to express the difference between the reconstruction losses of $\What_1$ and $\What_2$,}
\begin{equation}
\begin{aligned}
\bB_{12}&=(\What_1-\bW)^\top(\What_1-\bW)
-(\What_2-\bW)^\top(\What_2-\bW),\\
e_{12}(\bC)&=
\left[\lambda_{\max}\!\left(\bC^{-1/2}\bB_{12}\bC^{-1/2}\right)\right]_+.
\end{aligned}
\label{eq:pairwise-exposure}
\end{equation}
\settingrevision{For an input vector $u$, each squared output error measures the difference between a quantized layer's output and the full-precision output $\bW u$. The quantity $u^\top\bB_{12}u$ subtracts this error for $\What_2$ from that for $\What_1$. The shorthand $e_{12}(\bC)$ gives the largest increase in their reconstruction-loss difference under a distribution-change budget of one. Here $[a]_+=\max\{a,0\}$, since $\boldsymbol\Delta=0$ satisfies the constraints and leaves the difference unchanged. We will show that the loss of $\What_1$ remains no larger than that of $\What_2$ exactly when $\cL(\What_2)-\cL(\What_1)\ge\alpha e_{12}(\bC)$.}

\begin{proof}[Proof of \cref{thm:ordering-certificate}]
Under the second moment $\bH+\boldsymbol\Delta$, the difference between the two weight matrices' reconstruction losses is
\begin{equation}
\begin{aligned}
&\Tr\!\left[(\What_1-\bW)(\bH+\boldsymbol\Delta)(\What_1-\bW)^\top\right]\\
&\quad-\Tr\!\left[(\What_2-\bW)(\bH+\boldsymbol\Delta)(\What_2-\bW)^\top\right]\\
&\qquad=\settingrevision{\cL(\What_1)-\cL(\What_2)}+\Tr(\bB_{12}\boldsymbol\Delta),
\end{aligned}
\label{eq:pairwise-deployment-difference}
\end{equation}
The first term is the loss difference on calibration data. The second is the change in that difference caused by $\boldsymbol\Delta$.

Set $\mathbf E=\bC^{1/2}\boldsymbol\Delta\bC^{1/2}$ and $\mathbf D=\bC^{-1/2}\bB_{12}\bC^{-1/2}$. The constraints on $\boldsymbol\Delta$ are equivalent to $\mathbf E\succeq0$ and $\Tr(\mathbf E)\le\alpha$. Hence
\begin{equation}
\begin{split}
\max_{\boldsymbol\Delta\succeq0,\,
\Tr(\bC\boldsymbol\Delta)\le\alpha}
\Tr(\bB_{12}\boldsymbol\Delta)
&=\max_{\mathbf E\succeq0,\,\Tr(\mathbf E)\le\alpha}
\Tr(\mathbf D\mathbf E)\\
&=\alpha[\lambda_{\max}(\mathbf D)]_+.
\end{split}
\label{eq:pairwise-support-function}
\end{equation}
The last equality follows by setting $\mathbf E=\alpha vv^\top$ for a unit leading eigenvector $v$ when its eigenvalue is positive, and setting $\mathbf E=0$ otherwise. The largest loss difference in \cref{eq:pairwise-deployment-difference} is therefore \settingrevision{$\cL(\What_1)-\cL(\What_2)+\alpha e_{12}(\bC)$}. The loss of $\What_1$ remains no larger under every change satisfying the two constraints exactly when this expression is nonpositive. \settingrevision{For $\boldsymbol\Delta=\alpha uu^\top$ with $u^\top\bC u=1$, the change in the loss difference is $\alpha u^\top\bB_{12}u$. Requiring it to be at most $\cL(\What_2)-\cL(\What_1)$ for every such $u$ gives \cref{eq:directional-ranking-stability}. Since $\cL(\What_2)-\cL(\What_1)>0$ and a maximizing $\boldsymbol\Delta$ has rank at most one, the directional condition is equivalent to the condition over all matrices $\boldsymbol\Delta$ satisfying the constraints.}

\settingrevision{If the inequality fails, take a unit eigenvector $v$ corresponding to the largest eigenvalue of $\mathbf D$ and set $\boldsymbol\Delta^*=\alpha\bC^{-1/2}vv^\top\bC^{-1/2}$. This rank-one change satisfies both constraints and makes $\What_1$ have higher expected reconstruction loss than $\What_2$ on the new inputs.}
\end{proof}

\settingrevision{When $e_{12}(\bC)>0$, $\bigl(\cL(\What_2)-\cL(\What_1)\bigr)/e_{12}(\bC)$ is the largest distribution-change budget for which $\What_1$ is guaranteed to have no larger reconstruction loss than $\What_2$. If $e_{12}(\bC)=0$, no positive semidefinite change in the second-moment matrix can give $\What_1$ higher reconstruction loss than $\What_2$. The result compares layer-output reconstruction losses, not model PPL.}

\subsection{A fixed-grid example of higher loss on a new input}
\label{app:finite-bit-trap}

The following example gives two quantized weight vectors with arbitrarily small reconstruction losses on calibration data. The vector with the smaller loss nevertheless has a larger squared output error on a new input. Both vectors use the same fixed scale, and the new input has unit length, so the difference does not arise from changing the scale or increasing the input magnitude.

\begin{theorem}[Lower reconstruction loss on calibration data need not generalize]
\label{thm:finite-bit-trap}
Consider a linear layer with two inputs and one output, with reconstruction loss $\cL(\widehat{\mathbf w})=(\widehat{\mathbf w}-\mathbf w)\mathbf H(\widehat{\mathbf w}-\mathbf w)^\top$. Let $\mathcal Z_b$ be any finite set of integer codes containing $0$ and $1$, and fix the quantization scale $s>0$ in $\cQ_b(s)=\{s(k_1,k_2):k_1,k_2\in\mathcal Z_b\}$. For every $\varepsilon>0$, there exist a full-precision weight vector $\mathbf w\notin\cQ_b(s)$, a calibration second-moment matrix $\mathbf H\succ0$ with $\Tr(\mathbf H)=1$, and two quantized weight vectors $\widehat{\mathbf w}_{\rm rec},\widehat{\mathbf w}_{\rm alt}\in\cQ_b(s)$ such that $\widehat{\mathbf w}_{\rm rec}$ uniquely minimizes reconstruction loss over the grid and $0<\cL(\widehat{\mathbf w}_{\rm rec})<\cL(\widehat{\mathbf w}_{\rm alt})\le\varepsilon$. On the unit-length input $\mathbf x_*=(0,1)^\top$, their squared output errors are $s^2$ and $0$, respectively. The full-precision weight vector $\mathbf w$ is the unique unconstrained minimizer of the same reconstruction loss.
\end{theorem}

\begin{corollary}[Reconstruction-only refinement can increase loss on a new input]
\label{cor:initializer-degradation}
In the example of \cref{thm:finite-bit-trap}, take $\widehat{\mathbf w}_{\rm alt}$ as the initial quantized weights. Exactly minimizing reconstruction loss over the same grid returns $\widehat{\mathbf w}_{\rm rec}$. This reduces the loss on calibration data but increases the squared output error on $\mathbf x_*$ by exactly $s^2$.
\end{corollary}

\begin{proof}[Proof of \cref{thm:finite-bit-trap}]
Fix a finite set of integer codes $\mathcal Z_b$ containing $0$ and $1$. Choose an irrational $\delta\in(0,1)$ such that
\begin{equation}
k_1+\delta k_2\ne0
\quad\text{for all }(k_1,k_2)\in\mathcal Z_b^2\setminus\{(0,0)\}.
\label{eq:delta-choice}
\end{equation}
An irrational $\delta$ ensures that the expression above is nonzero for every nonzero integer pair in $\mathcal Z_b^2$.

Define $\mathbf w=(-\delta s,s)$, $\mathbf x=(1,\delta)^\top$, and $\widetilde{\mathbf H}_\eta=\mathbf x\mathbf x^\top+\eta\mathbf I_2$. Let $Z_\eta=1+\delta^2+2\eta$ and $\mathbf H_\eta=\widetilde{\mathbf H}_\eta/Z_\eta$, so $\Tr(\mathbf H_\eta)=1$. \settingrevision{We first use the unnormalized second-moment matrix $\widetilde{\mathbf H}_\eta$}, since dividing by the positive scalar $Z_\eta$ preserves the loss comparison between all weight vectors in the grid. Because $\mathbf w\mathbf x=0$, the reconstruction loss of $\widehat{\mathbf w}=s(k_1,k_2)$ becomes
\begin{equation}
L_\eta(\widehat{\mathbf w})
=s^2(k_1+\delta k_2)^2
+\eta\norm{\widehat{\mathbf w}-\mathbf w}_2^2.
\label{eq:trap-loss}
\end{equation}
At $\eta=0$, \cref{eq:delta-choice} makes $\widehat{\mathbf w}_{\mathrm{rec}}=(0,0)$ the only quantized weight vector with zero loss. The finite grid gives a positive reconstruction-loss gap
\begin{equation}
r=\min_{\widehat{\mathbf w}\in\cQ_b(s),\,\widehat{\mathbf w}\ne\widehat{\mathbf w}_{\mathrm{rec}}}
L_0(\widehat{\mathbf w})>0.
\label{eq:finite-grid-gap}
\end{equation}
Let $D=\max_{\widehat{\mathbf w}\in\cQ_b(s),\,\widehat{\mathbf w}\ne\widehat{\mathbf w}_{\mathrm{rec}}} |\norm{\widehat{\mathbf w}-\mathbf w}_2^2-\norm{\mathbf w}_2^2|$. Set $\eta_0=1$ if $D=0$ and $\eta_0=r/(2D)$ otherwise. For every nonzero quantized weight vector and $0<\eta<\eta_0$,
\begin{align}
L_\eta(\widehat{\mathbf w})-L_\eta(\widehat{\mathbf w}_{\mathrm{rec}})
&=L_0(\widehat{\mathbf w})+\eta
\left(\norm{\widehat{\mathbf w}-\mathbf w}_2^2-\norm{\mathbf w}_2^2\right)\\
&\ge r-\eta D>r/2>0.
\end{align}
\settingrevision{Thus $\widehat{\mathbf w}_{\mathrm{rec}}$ remains the unique minimizer of reconstruction loss over the grid when the second-moment matrix is positive definite.}

Take $\widehat{\mathbf w}_{\mathrm{alt}}=(0,s)$. The two weight vectors differ from $\mathbf w$ by $(\delta s,-s)$ and $(\delta s,0)$, respectively. Substitution gives
\begin{equation}
L_\eta(\widehat{\mathbf w}_{\mathrm{rec}})=\eta s^2(1+\delta^2),
\qquad
L_\eta(\widehat{\mathbf w}_{\mathrm{alt}})=\delta^2s^2(1+\eta),
\label{eq:trap-candidate-losses}
\end{equation}
whose difference is $s^2(\delta^2-\eta)$. The reconstruction minimizer has strictly smaller loss on calibration data whenever $\eta<\delta^2$.

Given $\varepsilon>0$, choose $\delta$ satisfying \cref{eq:delta-choice} and $\delta^2s^2\le\varepsilon/4$. Then choose
\begin{equation}
0<\eta<\min\left\{\eta_0,\delta^2/2,1,
\frac{\varepsilon}{4s^2(1+\delta^2)}\right\}.
\end{equation}
These choices imply $0<L_\eta(\widehat{\mathbf w}_{\mathrm{rec}})<L_\eta(\widehat{\mathbf w}_{\mathrm{alt}})\le\varepsilon$ for the unnormalized second moment. Normalizing its trace to one reduces both losses by the common factor $Z_\eta^{-1}$ and preserves the strict ordering. For the unit input vector $\mathbf x_*=(0,1)^\top$,
\begin{equation}
|(\widehat{\mathbf w}_{\mathrm{rec}}-\mathbf w)\mathbf x_*|^2=s^2,
\qquad
|(\widehat{\mathbf w}_{\mathrm{alt}}-\mathbf w)\mathbf x_*|^2=0.
\end{equation}
Finally, $\mathbf H_\eta\succ0$ because $\eta>0$. The unconstrained loss is zero only at $\widehat{\mathbf w}=\mathbf w$, which the grid cannot represent because $-\delta$ is not an integer. This second moment can also be realized by input activations. For example, a distribution that assigns probability $1/2$ to each column of $\mathbf X=\sqrt{2}\,\mathbf H_\eta^{1/2}$ has second moment $\mathbf H_\eta$. This completes the example with a fixed scale, a trace-one second moment, and a unit-length new input.
\end{proof}

\begin{proof}[Proof of \cref{cor:initializer-degradation}]
By \cref{thm:finite-bit-trap}, $\widehat{\mathbf w}_{\rm rec}$ uniquely minimizes reconstruction loss over the entire grid and has strictly lower loss on calibration data than $\widehat{\mathbf w}_{\rm alt}$. Exact minimization must therefore return $\widehat{\mathbf w}_{\rm rec}$. The squared output errors of these two vectors on $\mathbf x_*$ are $s^2$ and zero, respectively.
\end{proof}

Sampling independently from the two-point distribution above satisfies the assumptions of \cref{thm:population-persistence}. Taking $\cJ(\widehat{\mathbf w})=|(\widehat{\mathbf w}-\mathbf w)\mathbf x_*|^2$, the theorem gives an evaluation-loss gap of $s^2$ for the exact empirical minimizer for all sufficiently large sample sizes, with probability one. This is a counterexample for linear-layer reconstruction, rather than a bound on model PPL. It permits a general change in input distribution, whereas the following DRQ result uses the positive semidefinite changes specified in \cref{sec:ranking-stability}.

\subsection{Worst-case reconstruction loss}
\label{app:exact-risk}

\settingrevision{DRQ evaluates each quantized weight matrix by its largest expected reconstruction loss over input activation distributions whose second-moment matrices are $\bH+\boldsymbol\Delta$, with $\boldsymbol\Delta\succeq0$ and $\Tr(\bC\boldsymbol\Delta)\le\alpha$. We compute $\bH$ from the search subset and keep both $\bH$ and $\bC$ fixed during refinement. \Cref{thm:exact-risk} shows that a change along one input direction attains the largest loss increase. Finding this direction reduces to an eigenvalue problem, as shown below.}

\begin{proof}[Proof of \cref{thm:exact-risk}]
Set $\mathbf E=\bC^{1/2}\boldsymbol\Delta\bC^{1/2}$ and
\begin{equation*}
\mathbf B=\bC^{-1/2}(\What-\bW)^\top(\What-\bW)\bC^{-1/2}.
\end{equation*}
This change of variables is invertible, and the constraints become $\mathbf E\succeq0$ and $\Tr(\mathbf E)\le\alpha$. The expected reconstruction loss under the second-moment matrix $\bH+\boldsymbol\Delta$ is
\begin{equation}
\Tr\!\left[(\What-\bW)(\bH+\boldsymbol\Delta)(\What-\bW)^\top\right]
=\cL(\What)+\Tr(\mathbf B\mathbf E).
\end{equation}
Since both matrices are positive semidefinite, $\Tr(\mathbf B\mathbf E)\le\lambda_{\max}(\mathbf B)\Tr(\mathbf E)\le\alpha\lambda_{\max}(\mathbf B)$.

Let $v$ be a unit eigenvector for the largest eigenvalue of $\mathbf B$, and choose $\mathbf E^*=\alpha vv^\top$. Then $\Tr(\mathbf E^*)=\alpha$ and the upper bound is attained. The resulting \settingrevision{change in the second-moment matrix}
\begin{equation}
\boldsymbol\Delta^*=\alpha\bC^{-1/2}vv^\top\bC^{-1/2}
\end{equation}
\settingrevision{satisfies both constraints and uses the full distribution-change budget $\alpha$.} Finally, $\lambda_{\max}(\mathbf B)=\cS_{\bC}(\What)$ by setting $u=\bC^{-1/2}v$, which gives $u^\top\bC u=1$ and $(\What-\bW)^\top(\What-\bW)u=\lambda_{\max}(\mathbf B)\bC u$. This proves \cref{thm:exact-risk}.
\end{proof}

If $\alpha=0$ or $\What=\bW$, $\boldsymbol\Delta=0$ attains the maximum. Otherwise the construction above has rank one and can be realized by adding an independent perturbation that equals $\sqrt{\alpha}\,u$ or $-\sqrt{\alpha}\,u$, each with probability $1/2$, to the calibration input activations. Its mean is zero and its second moment is $\alpha uu^\top$, so the resulting distribution has second moment $\bH+\boldsymbol\Delta^*$.

The maximizing direction depends on both the weight difference $\What-\bW$ and $\bC$. Changing integer codes can therefore change the worst-case direction, which is why DRQ estimates it again before accepting proposed weights.

\subsection{Comparing refined and initial quantized weights}
\label{app:risk-certificate-bridge}

Two weight matrices can attain their worst-case losses under different input distributions. A smaller worst-case loss therefore does not imply smaller loss under every distribution in the set. To check this stronger property, we compare both matrices under the same $\boldsymbol\Delta$. Write $\What_0$ for the initial quantized weights. Using $e_{12}(\bC)$ from \cref{eq:pairwise-exposure}, define
\begin{equation}
\Gamma_\alpha(\What_1,\What_2)
=\cL(\What_1)-\cL(\What_2)
+\alpha e_{12}(\bC).
\label{eq:exact-robust-excess}
\end{equation}

\begin{proposition}[Comparing reconstruction losses under the same input distribution]
\label{prop:risk-certificate-bridge}
For any $\What_1,\What_2\in\R^{m\times d}$, $\bC\succ0$, and $\alpha\ge0$, the quantity in \cref{eq:exact-robust-excess} satisfies
\begin{equation}
\begin{aligned}
\Gamma_\alpha(\What_1,\What_2)
&=\max_{\boldsymbol\Delta\succeq0,\,\Tr(\bC\boldsymbol\Delta)\le\alpha}
\Tr\!\left(\bB_{12}(\bH+\boldsymbol\Delta)\right)\\
&\le\cL_{\rm DRQ}(\What_1)-\cL(\What_2).
\end{aligned}
\label{eq:risk-certificate-upper-bound}
\end{equation}
Consequently, $\Gamma_\alpha(\What_1,\What_0)<0$ certifies that $\What_1$ has strictly smaller reconstruction loss than $\What_0$ for every second moment satisfying these constraints.
\end{proposition}

\begin{proof}
The equality follows from \cref{eq:pairwise-support-function}. Since $\boldsymbol\Delta\succeq0$, the reconstruction loss of $\What_2$ is at least $\cL(\What_2)$, while that of $\What_1$ is at most $\cL_{\rm DRQ}(\What_1)$. Subtracting these bounds and maximizing proves \cref{eq:risk-certificate-upper-bound}. A negative maximum makes the pairwise difference negative for every second moment satisfying the constraints.
\end{proof}

\langrevision{The guarantee uses the exact value of $\Gamma_\alpha$. DRQ estimates this quantity for reporting when its selected weights differ from the initial weights. A negative estimate alone does not establish the guarantee, since a Rayleigh quotient can underestimate the largest eigenvalue. The acceptance rule uses the reconstruction tolerances and a decrease in the estimated worst-case reconstruction loss, as described in Appendix~\ref{app:algorithm-details}.}

\subsection{Equivalent optimization over input directions}
\label{app:equivalent-forms}

For a matrix, $\norm{\cdot}_2$ denotes its largest singular value. The identity $\norm{(\What-\bW)\bC^{-1/2}}_2^2=\max_{u^\top\bC u=1}\norm{(\What-\bW)u}_2^2$ gives the directional penalty in \cref{eq:marq-loss}. Introducing a scalar $t$ that bounds this error gives an equivalent form of the minimization in \cref{eq:marq-minmax},
\begin{equation}
\begin{aligned}
\min_{\What\in\cQ,t}\quad
&\cL(\What)+{\revision{\alpha}} t\\
\text{s.t.}\quad
&\norm{(\What-\bW)u}_2^2\le t,
\quad\forall u:u^\top\bC u=1.
\end{aligned}
\label{eq:semi-infinite}
\end{equation}
For fixed weights and $\alpha>0$, minimizing over $t$ sets it to the largest directional error. At $\alpha=0$, this value of $t$ is still feasible and optimal, although larger values are also optimal. The formulation therefore preserves the minimum and the minimizing weight matrices for every $\alpha\ge0$. It is called semi-infinite because it has finitely many optimization variables but one constraint for each of infinitely many input directions.

\subsection{A reconstruction bound using Wasserstein distance}
\label{app:wasserstein}

\settingrevision{Another way to restrict a change in the input activation distribution is to pair original and new activation vectors and bound their average squared difference. Using the weighted 2-Wasserstein distance to make this comparison gives a reconstruction-loss bound involving the same worst-case directional error as DRQ.}

Let $\widehat P_n$ assign probability $1/n$ to each calibration input vector used for code search, so its second moment is $\bH$. For a distribution $Q$ with a finite second moment, define the weighted 2-Wasserstein distance by
\[
\mathsf W_{2,\bC}^2(Q,\widehat P_n)
=\inf_{\pi\in\Pi(\widehat P_n,Q)}
\mathbb E_{(x,x')\sim\pi}\norm{\bC^{1/2}(x'-x)}_2^2.
\]
\settingrevision{Each joint distribution in $\Pi(\widehat P_n,Q)$ pairs an original activation vector $x$ with a new vector $x'$, while preserving their distributions $\widehat P_n$ and $Q$. The squared distance is the smallest average value of $\norm{\bC^{1/2}(x'-x)}_2^2$ over these pairings. A change $h=x'-x$ contributes $h^\top\bC h$ to this average.} For $\varepsilon\ge0$ and $\zeta>0$,
\begin{equation}
\begin{aligned}
&\sup_{Q:\mathsf W_{2,\bC}(Q,\widehat P_n)\le\varepsilon}
\mathbb E_Q\norm{(\What-\bW)x}_2^2\\
&\qquad\le(1+\zeta)\left[\cL(\What)
+\frac{\varepsilon^2}{\zeta}\norm{(\What-\bW)\bC^{-1/2}}_2^2\right].
\end{aligned}
\label{eq:wasserstein-bound}
\end{equation}
The bound combines reconstruction loss on calibration data with the same worst-case directional error used by DRQ. It is an upper bound for distributions within a Wasserstein distance, whereas \cref{thm:exact-risk} gives an equality for DRQ's second-moment constraints. These are two different sets of distributions.

\begin{proof}[Proof of \cref{eq:wasserstein-bound}]
Fix $Q$ with $\mathsf W_{2,\bC}(Q,\widehat P_n)\le\varepsilon$ and let $(x,x')$ be drawn from a joint distribution $\pi\in\Pi(\widehat P_n,Q)$. Minkowski's inequality in $L_2$ gives
\begin{equation}
\begin{aligned}
&\left(\mathbb E_Q\norm{(\What-\bW)x'}_2^2\right)^{1/2}\\
&\quad\le\left(\mathbb E_{\widehat P_n}\norm{(\What-\bW)x}_2^2\right)^{1/2}
+\left(\mathbb E_\pi\norm{(\What-\bW)(x'-x)}_2^2\right)^{1/2}.
\end{aligned}
\end{equation}
For every displacement $h$, $\norm{(\What-\bW)h}_2\le\norm{(\What-\bW)\bC^{-1/2}}_2 \norm{\bC^{1/2}h}_2$. \settingrevision{Taking the infimum over joint distributions $\pi$} bounds the second term by $\varepsilon\norm{(\What-\bW)\bC^{-1/2}}_2$. \settingrevision{The first term is} $\cL(\What)^{1/2}$. Applying $(a+b)^2\le(1+\zeta)a^2+(1+\zeta^{-1})b^2$ and \settingrevision{taking the supremum over distributions $Q$ within the specified Wasserstein distance} proves \cref{eq:wasserstein-bound}.
\end{proof}

%% file: tex/appendix/C_algorithm_details.tex
\section{Algorithmic Details}
\label{app:algorithm-details}
\label{app:reproducibility}
\label{app:protocol}

\paragraph{Blockwise integration.}
DRQ follows Algorithm~\ref{alg:drq-overview}, refining each linear layer after the base PTQ method quantizes it. The bit width, grouping, scales, and zero points remain fixed. DRQ uses cached calibration activations and the corresponding full-precision weights after any scaling or rotation by the base method. In GPTQ, refined layers supply inputs for subsequent layer groups within the block. In AWQ, refinement follows the block's scaling and clipping. The refined block then produces calibration inputs for the next block.

\paragraph{Calibration subsets.}
We split the selected activation vectors into a search subset $\bX_{\rm d}$ and two separate checking subsets $\bX_{\rm c}^{(1)}$ and $\bX_{\rm c}^{(2)}$. The search subset determines $\bH$, $\bC$, and reconstruction loss used to score code changes. These subsets and matrices stay fixed during refinement. Appendix~\ref{app:moe-experts} describes the adjustments for experts receiving few calibration tokens.

\paragraph{Proposing code changes.}
Each round estimates the current weights' worst-case input direction $u$. Holding $u$ fixed, we score code replacements by their change in $\cL(\What)+\alpha\norm{(\What-\bW)u}_2^2$. In each row, we retain the eight lowest-scoring single-code replacements and test pairs that change two different weights. We keep the single replacement or pair with the largest decrease in this objective, provided the decrease is positive. We then rank these row updates by their decreases and apply different fractions of the list to the current weights, starting with the largest decreases.

\paragraph{Accepting new weights.}
For each proposed matrix, we estimate its worst-case direction again. Its estimated $\cL_{\rm DRQ}$ must decrease relative to the current weights. On the search subset and each checking subset, reconstruction loss must also stay within the specified percentage above that of the initial PTQ weights $\What_0$. Proposals that fail these conditions or the numerical checks are rejected. We continue with the passing matrix having the lowest estimated $\cL_{\rm DRQ}$, or stop if none passes. We return the initial or accepted weights with the lowest estimated loss.

\paragraph{Main parameters.}
For the dense-model and six-method comparisons, we select $\alpha$ and $\beta$ using 64 validation samples from Pile \citep{gao2020pile}. The settings below use the standard $128\times2{,}048$ calibration setup. Other comparisons retain their stated settings. Reconstruction-only refinement uses the same procedure with $\alpha=0$.

\begin{center}
\small
\setlength{\tabcolsep}{5pt}
\renewcommand{\arraystretch}{1.12}
\begin{tabular}{@{}p{0.43\linewidth}p{0.53\linewidth}@{}}
\toprule
Parameter & Setting \\
\midrule
Activation vectors per layer & 512 for search; $2\times256$ for checking \\
Distribution-change budget $\alpha$ & $\{0.1,0.3,0.5,1\}$ \\
Identity weight $\beta$ in $\bC$ & $\{0.005,0.1,0.25,0.75,1\}$ \\
Reconstruction-loss increase limit & 3\% for GPTQ; 1\% for other base methods \\
Single-code replacements retained per row & 8 \\
Fractions of ranked row updates tested & $\{0,1/32,1/16,1/8,1/4,1/2,1\}$ \\
Maximum refinement rounds & 3 \\
\bottomrule
\end{tabular}
\end{center}
\FloatBarrier

%% file: tex/appendix/D_00_experimental_results.tex
\FloatBarrier
\Needspace{20\baselineskip}
\section{More Experimental Results}
\label{app:additional-results}

\input{tex/appendix/D_04_dense_results}
\input{tex/appendix/D_05_ptq_methods}
\input{tex/appendix/D_06_moe_experts}
\input{tex/appendix/D_07_parameter_sensitivity}

%% file: tex/appendix/D_04_dense_results.tex
\subsection{Detailed dense Llama results}
\label{app:dense-results}

\settingrevision{\Crefrange{tab:llama32-1b-instruct-task-accuracy}{tab:llama31-70b-instruct-task-accuracy} report accuracy on each of the seven downstream tasks before and after DRQ refinement. Average task accuracy improves in every reported comparison with the initial quantized model. The size of the improvement at lower bit widths depends on the base PTQ method. For the 70B model quantized with GPTQ at W2, DRQ improves all seven task scores and raises average accuracy by 6.33 percentage points.}

\begin{table}[H]
\centering
\begingroup
\definecolor{MockBlue}{HTML}{176B91}
\definecolor{MockInk}{HTML}{202B35}
\definecolor{MockMuted}{HTML}{5E6873}

\caption{\settingrevision{Llama-3.2-1B-Instruct accuracy (\%) on each downstream task.}}
\label{tab:llama32-1b-instruct-task-accuracy}
\fontsize{8.1}{10.1}\selectfont
\setlength{\tabcolsep}{2pt}
\renewcommand{\arraystretch}{1.18}
\noindent
\begin{tabular*}{\linewidth}{@{\extracolsep{\fill}}cl*{8}{c}@{}}
\toprule
Bits & Method & PIQA & HellaSwag & MMLU & BoolQ & ARC-C & ARC-E & WinoGrande & Mean \\
\midrule
FP & Unquantized & 74.16 & 60.76 & 45.86 & 69.42 & 38.05 & 63.26 & 59.27 & 58.68 \\
\midrule
\multirow[c]{4}{*}[-1.5pt]{W2} & GPTQ & 51.52 & 26.47 & 24.19 & 42.54 & 24.66 & 25.63 & 48.22 & 34.75 \\
 & \textcolor{MockBlue}{\quad + DRQ} & \textbf{51.80} & \textbf{26.79} & \textbf{26.41} & \textbf{43.15} & \textbf{25.77} & \textbf{26.22} & \textbf{51.38} & \textbf{35.93} \\
\addlinespace[3pt]
 & AWQ & 51.74 & 25.99 & 23.66 & 47.65 & 25.68 & 23.53 & 49.49 & 35.39 \\
 & \textcolor{MockBlue}{\quad + DRQ} & 51.31 & \textbf{26.93} & \textbf{24.86} & \textbf{51.25} & 25.51 & \textbf{24.20} & \textbf{51.54} & \textbf{36.51} \\
\midrule
\multirow[c]{4}{*}[-1.5pt]{W3} & GPTQ & 67.90 & 52.69 & 28.32 & 62.26 & 32.68 & 53.16 & 56.51 & 50.50 \\
 & \textcolor{MockBlue}{\quad + DRQ} & 67.79 & \textbf{53.60} & \textbf{33.75} & \textbf{62.94} & \textbf{34.22} & \textbf{53.75} & 55.72 & \textbf{51.68} \\
\addlinespace[3pt]
 & AWQ & 68.44 & 54.08 & 35.59 & 66.42 & 32.17 & 55.93 & 55.88 & 52.65 \\
 & \textcolor{MockBlue}{\quad + DRQ} & \textbf{68.88} & \textbf{54.13} & \textbf{36.25} & \textbf{66.79} & \textbf{34.13} & 54.80 & \textbf{56.59} & \textbf{53.08} \\
\midrule
\multirow[c]{4}{*}[-1.5pt]{W4} & GPTQ & 73.72 & 58.99 & 43.95 & 67.95 & 36.86 & 61.95 & 58.88 & 57.47 \\
 & \textcolor{MockBlue}{\quad + DRQ} & \textbf{73.78} & \textbf{59.36} & \textbf{44.28} & \textbf{68.10} & \textbf{37.37} & \textbf{63.22} & \textbf{60.77} & \textbf{58.13} \\
\addlinespace[3pt]
 & AWQ & 72.63 & 58.74 & 43.41 & 68.72 & 36.95 & 61.99 & 58.48 & 57.27 \\
 & \textcolor{MockBlue}{\quad + DRQ} & 71.87 & \textbf{59.02} & 43.09 & 68.59 & \textbf{38.23} & \textbf{62.33} & \textbf{58.88} & \textbf{57.43} \\
\bottomrule
\end{tabular*}
\endgroup
\end{table}

\begin{table}[H]
\centering
\begingroup
\definecolor{MockBlue}{HTML}{176B91}
\definecolor{MockInk}{HTML}{202B35}
\definecolor{MockMuted}{HTML}{5E6873}

\caption{\settingrevision{Llama-3.2-3B-Instruct accuracy (\%) on each downstream task.}}
\label{tab:llama32-3b-instruct-task-accuracy}
\fontsize{8.1}{10.1}\selectfont
\setlength{\tabcolsep}{2pt}
\renewcommand{\arraystretch}{1.18}
\noindent
\begin{tabular*}{\linewidth}{@{\extracolsep{\fill}}cl*{8}{c}@{}}
\toprule
Bits & Method & PIQA & HellaSwag & MMLU & BoolQ & ARC-C & ARC-E & WinoGrande & Mean \\
\midrule
FP & Unquantized & 75.68 & 70.47 & 60.45 & 77.74 & 45.56 & 67.76 & 67.72 & 66.48 \\
\midrule
\multirow[c]{4}{*}[-1.5pt]{W2} & GPTQ & 51.69 & 27.74 & 23.86 & 40.24 & 25.09 & 26.01 & 47.43 & 34.58 \\
 & \textcolor{MockBlue}{\quad + DRQ} & 50.60 & \textbf{29.11} & \textbf{24.83} & \textbf{40.43} & 25.00 & 25.34 & \textbf{49.88} & \textbf{35.03} \\
\addlinespace[3pt]
 & AWQ & 51.31 & 26.76 & 26.89 & 42.66 & 25.26 & 26.09 & 48.70 & 35.38 \\
 & \textcolor{MockBlue}{\quad + DRQ} & \textbf{51.36} & \textbf{28.44} & 26.89 & \textbf{60.67} & 24.74 & \textbf{26.30} & \textbf{50.28} & \textbf{38.38} \\
\midrule
\multirow[c]{4}{*}[-1.5pt]{W3} & GPTQ & 69.15 & 61.76 & 50.04 & 63.61 & 33.53 & 52.65 & 60.06 & 55.83 \\
 & \textcolor{MockBlue}{\quad + DRQ} & \textbf{71.33} & \textbf{65.25} & 49.25 & \textbf{73.33} & \textbf{38.57} & \textbf{59.60} & \textbf{64.33} & \textbf{60.24} \\
\addlinespace[3pt]
 & AWQ & 72.96 & 65.02 & 51.59 & 73.18 & 39.85 & 64.73 & 64.48 & 61.69 \\
 & \textcolor{MockBlue}{\quad + DRQ} & \textbf{73.23} & \textbf{65.05} & \textbf{52.37} & 72.72 & \textbf{41.89} & \textbf{66.67} & \textbf{65.82} & \textbf{62.54} \\
\midrule
\multirow[c]{4}{*}[-1.5pt]{W4} & GPTQ & 75.14 & 69.93 & 58.26 & 76.15 & 44.80 & 66.92 & 65.67 & 65.26 \\
 & \textcolor{MockBlue}{\quad + DRQ} & \textbf{75.57} & 69.89 & \textbf{59.22} & \textbf{76.42} & \textbf{45.39} & \textbf{69.15} & \textbf{66.06} & \textbf{65.96} \\
\addlinespace[3pt]
 & AWQ & 75.03 & 69.70 & 59.09 & 78.72 & 46.08 & 67.72 & 66.93 & 66.18 \\
 & \textcolor{MockBlue}{\quad + DRQ} & \textbf{75.57} & \textbf{69.85} & 59.03 & \textbf{79.14} & \textbf{46.16} & \textbf{69.19} & \textbf{67.96} & \textbf{66.70} \\
\bottomrule
\end{tabular*}
\endgroup
\end{table}

\begin{table}[H]
\centering
\begingroup
\definecolor{MockBlue}{HTML}{176B91}
\definecolor{MockInk}{HTML}{202B35}
\definecolor{MockMuted}{HTML}{5E6873}

\caption{\settingrevision{Llama-3.1-8B-Instruct accuracy (\%) on each downstream task.}}
\label{tab:llama31-8b-instruct-task-accuracy}
\fontsize{8.1}{10.1}\selectfont
\setlength{\tabcolsep}{2pt}
\renewcommand{\arraystretch}{1.18}
\noindent
\begin{tabular*}{\linewidth}{@{\extracolsep{\fill}}cl*{8}{c}@{}}
\toprule
Bits & Method & PIQA & HellaSwag & MMLU & BoolQ & ARC-C & ARC-E & WinoGrande & Mean \\
\midrule
FP & Unquantized & 81.12 & 79.25 & 68.05 & 83.94 & 55.03 & 79.59 & 73.80 & 74.40 \\
\midrule
\multirow[c]{4}{*}[-1.5pt]{W2} & GPTQ & 49.95 & 27.67 & 24.11 & 43.03 & 26.37 & 27.53 & 49.80 & 35.49 \\
 & \textcolor{MockBlue}{\quad + DRQ} & \textbf{51.80} & \textbf{29.11} & \textbf{24.60} & \textbf{47.43} & 25.85 & 26.56 & \textbf{51.14} & \textbf{36.64} \\
\addlinespace[3pt]
 & AWQ & 51.03 & 26.60 & 26.89 & 61.93 & 24.40 & 25.00 & 48.78 & 37.80 \\
 & \textcolor{MockBlue}{\quad + DRQ} & 51.03 & \textbf{26.67} & 26.89 & 60.28 & \textbf{24.74} & \textbf{25.21} & \textbf{51.22} & \textbf{38.01} \\
\midrule
\multirow[c]{4}{*}[-1.5pt]{W3} & GPTQ & 76.93 & 75.40 & 57.01 & 81.53 & 45.48 & 71.25 & 71.59 & 68.46 \\
 & \textcolor{MockBlue}{\quad + DRQ} & \textbf{78.02} & 74.90 & \textbf{58.43} & 81.19 & \textbf{47.95} & \textbf{74.66} & 71.19 & \textbf{69.48} \\
\addlinespace[3pt]
 & AWQ & 78.13 & 74.53 & 59.54 & 81.28 & 49.91 & 73.65 & 72.22 & 69.90 \\
 & \textcolor{MockBlue}{\quad + DRQ} & \textbf{78.24} & 74.44 & \textbf{59.97} & \textbf{81.62} & 49.74 & \textbf{74.03} & \textbf{72.69} & \textbf{70.10} \\
\midrule
\multirow[c]{4}{*}[-1.5pt]{W4} & GPTQ & 79.98 & 78.39 & 66.39 & 82.81 & 52.82 & 77.65 & 73.72 & 73.11 \\
 & \textcolor{MockBlue}{\quad + DRQ} & \textbf{80.20} & \textbf{78.40} & \textbf{66.96} & \textbf{83.36} & \textbf{55.20} & \textbf{79.21} & \textbf{74.66} & \textbf{74.00} \\
\addlinespace[3pt]
 & AWQ & 80.90 & 78.68 & 65.71 & 83.30 & 53.67 & 79.38 & 72.77 & 73.49 \\
 & \textcolor{MockBlue}{\quad + DRQ} & 80.85 & 78.57 & \textbf{66.02} & \textbf{83.46} & 53.67 & \textbf{79.63} & \textbf{73.16} & \textbf{73.62} \\
\bottomrule
\end{tabular*}
\endgroup
\end{table}

\begin{table}[H]
\centering
\begingroup
\definecolor{MockBlue}{HTML}{176B91}
\definecolor{MockInk}{HTML}{202B35}
\definecolor{MockMuted}{HTML}{5E6873}

\caption{\settingrevision{Llama-3.1-70B-Instruct accuracy (\%) on each downstream task.}}
\label{tab:llama31-70b-instruct-task-accuracy}
\fontsize{8.1}{10.1}\selectfont
\setlength{\tabcolsep}{2pt}
\renewcommand{\arraystretch}{1.18}
\noindent
\begin{tabular*}{\linewidth}{@{\extracolsep{\fill}}cl*{8}{c}@{}}
\toprule
Bits & Method & PIQA & HellaSwag & MMLU & BoolQ & ARC-C & ARC-E & WinoGrande & Mean \\
\midrule
FP & Unquantized & 83.73 & 84.63 & 82.29 & 88.01 & 63.23 & 83.67 & 78.93 & 80.64 \\
\midrule
\multirow[c]{4}{*}[-1.5pt]{W2} & GPTQ & 63.71 & 60.43 & 29.01 & 58.50 & 28.67 & 40.61 & 56.67 & 48.23 \\
 & \textcolor{MockBlue}{\quad + DRQ} & \textbf{67.08} & \textbf{64.81} & \textbf{35.14} & \textbf{68.56} & \textbf{34.22} & \textbf{51.64} & \textbf{60.46} & \textbf{54.56} \\
\addlinespace[3pt]
 & AWQ & 51.52 & 26.66 & 24.65 & 62.17 & 26.71 & 24.75 & 49.49 & 37.99 \\
 & \textcolor{MockBlue}{\quad + DRQ} & 51.31 & 26.61 & 24.65 & 62.17 & 26.54 & \textbf{25.04} & \textbf{49.72} & \textbf{38.01} \\
\midrule
\multirow[c]{4}{*}[-1.5pt]{W3} & GPTQ & 82.70 & 82.90 & 79.86 & 86.57 & 59.30 & 81.78 & 77.58 & 78.67 \\
 & \textcolor{MockBlue}{\quad + DRQ} & 81.50 & \textbf{82.99} & 79.85 & 86.54 & \textbf{61.60} & \textbf{83.08} & \textbf{77.66} & \textbf{79.03} \\
\addlinespace[3pt]
 & AWQ & 82.32 & 82.70 & 79.84 & 87.09 & 60.58 & 82.87 & 77.74 & 79.02 \\
 & \textcolor{MockBlue}{\quad + DRQ} & \textbf{82.86} & 82.63 & 79.60 & \textbf{87.28} & \textbf{61.43} & 82.79 & \textbf{77.98} & \textbf{79.23} \\
\midrule
\multirow[c]{4}{*}[-1.5pt]{W4} & GPTQ & 83.57 & 84.21 & 81.56 & 87.89 & 62.80 & 83.08 & 78.45 & 80.22 \\
 & \textcolor{MockBlue}{\quad + DRQ} & \textbf{83.73} & \textbf{84.39} & \textbf{81.80} & 87.74 & 62.80 & \textbf{83.46} & \textbf{79.95} & \textbf{80.55} \\
\addlinespace[3pt]
 & AWQ & 83.24 & 84.26 & 81.64 & 87.55 & 63.31 & 83.67 & 79.01 & 80.38 \\
 & \textcolor{MockBlue}{\quad + DRQ} & \textbf{83.79} & \textbf{84.29} & 81.43 & \textbf{87.71} & 62.97 & \textbf{83.92} & \textbf{79.56} & \textbf{80.52} \\
\bottomrule
\end{tabular*}
\endgroup
\end{table}

\FloatBarrier

%% file: tex/appendix/D_05_ptq_methods.tex
\subsection{Results across PTQ methods}
\label{app:cross-host-results}

\settingrevision{\Cref{tab:cross-host-task-accuracy} reports WT2 and C4 perplexity and accuracy on each of the seven downstream tasks for the initial and DRQ-refined Llama-3.2-3B-Instruct W3A16 models in \cref{fig:cross-host-radar}. The OmniQuant runs use learned weight clipping without learned equivalent transformations. Across the six selected comparisons, DRQ lowers both perplexities and raises average task accuracy. For GPTQ and AWQ, accuracy improves on six of the seven tasks. For RTN, gains on PIQA, BoolQ, and ARC-Easy outweigh declines on the other four tasks.}

\begin{table}[!htbp]
\centering
\input{tex/tables/task_accuracy.tex}
\end{table}
\FloatBarrier

\paragraph{QuaRot with GPTQ.}
\label{app:quarot-gptq}
We further evaluate DRQ on Llama-3.1-8B-Instruct with QuaRot's Hadamard rotations followed by GPTQ \citep{Ashkboos2024quarot,frantar2023gptq}. DRQ refines the integer codes of the quantized weights in the rotated representation while retaining the rotations and quantization parameters. We test W2A16 and W3A16 with a group size of 128. Each calibration set contains 128 WT2 sequences of 2,048 tokens. The reconstruction tolerances on the search and checking subsets are both 0.01. We use $(\alpha,\beta)=(1,0.005)$ for W2 and $(0.1,0.005)$ for W3. As is shown in \Cref{tab:quarot-gptq-best}, at W3, DRQ improves accuracy on all seven tasks, raising mean accuracy from 68.44\% to 69.72\% while lowering WT2 and C4 PPL. This shows that DRQ can further improve a model quantized with QuaRot and GPTQ.

\begin{table}[!htbp]
\centering
\input{tex/tables/best_results.tex}
\end{table}
\FloatBarrier

%% file: tex/tables/task_accuracy.tex
\begingroup
\definecolor{ResultBlue}{HTML}{176B91}
\definecolor{ResultInk}{HTML}{202B35}

\caption{Perplexity and seven-task accuracy (\%) for Llama-3.2-3B W3.}
\label{tab:cross-host-task-accuracy}
\fontsize{7.8}{9.8}\selectfont
\setlength{\tabcolsep}{1.5pt}
\renewcommand{\arraystretch}{1.18}
\begin{tabular*}{\linewidth}{@{\extracolsep{\fill}}l*{10}{c}@{}}
\toprule
& \multicolumn{2}{c}{Perplexity} & \multicolumn{8}{c}{Accuracy (\%)} \\
\cmidrule(lr){2-3}\cmidrule(lr){4-11}
Method & WT2 & C4 & PIQA & HellaSwag & MMLU & BoolQ & ARC-C & ARC-E & WinoGrande & Mean \\
\midrule
RTN & 17.97 & 26.96 & 69.42 & 62.41 & 45.21 & 72.39 & 37.12 & 56.78 & 62.51 & 57.98 \\
\textcolor{ResultBlue}{\quad + DRQ} & \textbf{17.60} & \textbf{26.12} & \textbf{70.18} & 62.03 & 45.20 & \textbf{73.58} & 36.95 & \textbf{57.95} & 62.04 & \textbf{58.28} \\
\addlinespace[3pt]
GPTQ & 15.74 & 30.34 & 69.15 & 61.76 & 50.04 & 63.61 & 33.53 & 52.65 & 60.06 & 55.83 \\
\textcolor{ResultBlue}{\quad + DRQ} & \textbf{14.63} & \textbf{22.11} & \textbf{71.33} & \textbf{65.25} & 49.25 & \textbf{73.33} & \textbf{38.57} & \textbf{59.60} & \textbf{64.33} & \textbf{60.24} \\
\addlinespace[3pt]
AWQ & 15.19 & 20.80 & 72.96 & 65.02 & 51.59 & 73.18 & 39.85 & 64.73 & 64.48 & 61.69 \\
\textcolor{ResultBlue}{\quad + DRQ} & \textbf{14.90} & \textbf{20.52} & \textbf{73.23} & \textbf{65.05} & \textbf{52.37} & 72.72 & \textbf{41.89} & \textbf{66.67} & \textbf{65.82} & \textbf{62.54} \\
\addlinespace[3pt]
OmniQuant & 14.52 & 21.23 & 71.60 & 63.99 & 50.28 & 73.94 & 37.20 & 62.21 & 61.17 & 60.06 \\
\textcolor{ResultBlue}{\quad + DRQ} & \textbf{14.44} & \textbf{21.12} & \textbf{71.87} & \textbf{64.01} & \textbf{50.41} & \textbf{74.22} & 37.12 & \textbf{62.84} & 60.93 & \textbf{60.20} \\
\addlinespace[3pt]
GPTAQ & 14.94 & 20.97 & 73.78 & 64.57 & 46.80 & 74.10 & 38.48 & 62.75 & 65.04 & 60.79 \\
\textcolor{ResultBlue}{\quad + DRQ} & \textbf{14.44} & \textbf{20.72} & \textbf{74.37} & 64.49 & \textbf{47.08} & 73.70 & \textbf{38.74} & \textbf{62.88} & 65.04 & \textbf{60.90} \\
\addlinespace[3pt]
ParoQuant & 13.56 & 20.31 & 74.10 & 67.24 & 54.76 & 74.50 & 40.61 & 64.23 & 65.35 & 62.97 \\
\textcolor{ResultBlue}{\quad + DRQ} & \textbf{13.38} & \textbf{19.94} & \textbf{75.03} & \textbf{67.78} & 54.54 & \textbf{75.38} & 40.61 & \textbf{64.73} & \textbf{66.06} & \textbf{63.45} \\
\bottomrule
\end{tabular*}
\endgroup

%% file: tex/tables/best_results.tex
\begingroup

\caption{QuaRot + GPTQ with DRQ on Llama-3.1-8B-Instruct.}
\label{tab:quarot-gptq-best}
\fontsize{7.3}{9.3}\selectfont
\setlength{\tabcolsep}{1.3pt}
\renewcommand{\arraystretch}{1.18}
\begin{tabular*}{\linewidth}{@{\extracolsep{\fill}}ll*{10}{c}@{}}
\toprule
& & \multicolumn{2}{c}{Perplexity} & \multicolumn{8}{c}{Accuracy (\%)} \\
\cmidrule(lr){3-4}\cmidrule(lr){5-12}
Bits & Method & WT2 & C4 & PIQA & HellaSwag & MMLU & BoolQ & ARC-C & ARC-E & WinoGrande & Mean \\
\midrule
\multirow{2}{*}{W2} & QuaRot + GPTQ & 72.62 & 164.40 & 53.05 & 30.91 & 24.11 & 44.59 & 22.95 & 29.46 & 49.57 & 36.38 \\
& \textcolor{mockblue}{\quad + DRQ} & \textbf{67.53} & \textbf{159.03} & \textbf{53.32} & \textbf{31.06} & 23.33 & \textbf{46.67} & \textbf{24.23} & \textbf{29.76} & 49.41 & \textbf{36.82} \\
\addlinespace[3pt]
\multirow{2}{*}{W3} & QuaRot + GPTQ & 8.58 & 14.86 & 78.18 & 74.07 & 58.60 & 81.44 & 46.67 & 69.95 & 70.17 & 68.44 \\
& \textcolor{mockblue}{\quad + DRQ} & \textbf{8.24} & \textbf{14.23} & \textbf{78.35} & \textbf{75.41} & \textbf{60.18} & \textbf{82.63} & \textbf{48.81} & \textbf{71.25} & \textbf{71.43} & \textbf{69.72} \\
\bottomrule
\end{tabular*}
\endgroup

%% file: tex/appendix/D_06_moe_experts.tex
\subsection{Calibration tokens and expert refinement}
\label{app:moe-experts}

We examine how the number of calibration tokens reaching each expert affects DRQ refinement in the Qwen3-30B-A3B run reported in \cref{tab:mock-moe}. We quantize 18,432 expert linear layers and 192 attention linear layers across 48 Transformer blocks, leaving the routers unquantized. Each expert contains three linear layers.

\paragraph{Routing distributes calibration tokens unevenly.}
\settingrevision{\Cref{fig:moe-expert-coverage}(a) shows that 4,794 of 18,624 linear layers (25.74\%) receive fewer than the 1,024 input activation vectors used for the full search and checking subsets. These include 399 linear layers with one or two routed tokens and 201 with none, all belonging to experts. The router sends different numbers of calibration tokens to different experts. Some experts therefore receive too few tokens to form separate checking subsets.}

\begin{figure}[!htbp]
\centering
\setlength{\parskip}{0pt}
\setlength{\abovecaptionskip}{-8pt}
\includegraphics[width=\linewidth]{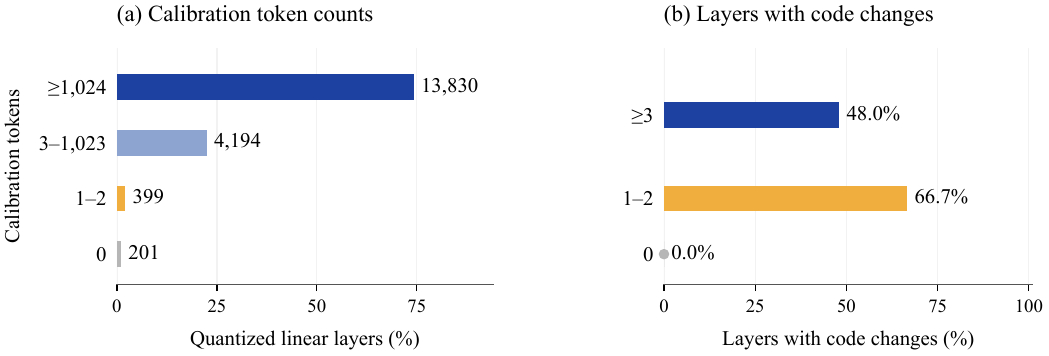}
\caption{Calibration token counts and integer-code changes in Qwen3-30B-A3B W3A16. (a) The 18,624 quantized linear layers grouped by the number of calibration tokens received, with layer counts shown beside the bars. (b) The fraction of linear layers with at least one changed integer code in each token-count group. The 0--2-token groups contain expert linear layers exclusively, while the $\geq3$ group includes both expert and attention linear layers.}
\label{fig:moe-expert-coverage}
\end{figure}

\paragraph{Refinement uses the available tokens.}
\settingrevision{DRQ adjusts its reconstruction checks to the number of available tokens. With at least three tokens, it forms three disjoint subsets, one for proposing and scoring integer-code changes and two for checking reconstruction loss. It reduces their sizes when needed. With one or two tokens, all input activation vectors are used for search. Linear layers receiving no tokens use RTN initialization and set $\bH$ and $\bC$ to the identity matrix. In this case, the reconstruction term is the squared Frobenius distance between the quantized and full-precision weights. Layers with fewer than three tokens have no separate checking subsets. The reconstruction-loss limits on both checking subsets therefore apply to layers with at least three tokens. This procedure uses the activations routed by the model without adding artificial tokens or changing which experts the router selects.}

In \cref{fig:moe-expert-coverage}(b), DRQ modifies 266 of the 399 linear layers receiving one or two tokens (66.7\%), compared with 8,657 of 18,024 linear layers in the $\geq3$ group (48.0\%). \settingrevision{Layers receiving one or two tokens are modified more often} but contribute only 199,289 of the 5,425,802 integer-code changes (3.67\%). All 201 zero-token linear layers retain their initial integer codes. Overall, DRQ changes 0.01815\% of the weights included in these counts, \settingrevision{compared with the codes before each layer is refined}, while retaining scales and zero points. These changes affect many linear layers but a small fraction of their weights, with most code changes occurring in the $\geq3$ group.

\FloatBarrier

%% file: tex/appendix/D_07_parameter_sensitivity.tex
\FloatBarrier
\Needspace{42\baselineskip}
\subsection{Parameter sensitivity}
\label{app:parameter-sensitivity}

\Cref{fig:parameter-sensitivity} evaluates the roles of $\alpha$ and $\beta$ on Llama-3.2-3B-Instruct W3. In \cref{eq:marq-loss}, $\alpha$ weights the worst-case directional error, while $\beta$ sets the weight on the identity matrix in $\bC$. We test all combinations of $\alpha\in\{0.1,0.3,0.5,1\}$ and $\beta\in\{0.005,0.1,0.25,0.75,1\}$. Among the displayed runs, 17 of 20 GPTQ configurations and 11 of 20 AWQ configurations improve average task accuracy and both WT2 and C4 PPL relative to the initial quantized model. Improvements on all three metrics occur across more GPTQ configurations, while the favorable AWQ settings concentrate around $\alpha=0.3$ and $0.5$.

At $\alpha=0.3$, both methods improve all three metrics across all five $\beta$ values. Increasing $\alpha$ to 1 worsens WT2 PPL in three GPTQ and two AWQ configurations. A stronger penalty on worst-case directional error therefore does not consistently improve model performance in this sweep. The objective balances this penalty against reconstruction loss on calibration data, and the parameter settings that improve model performance depend on the base PTQ method.

\begin{figure}[!htbp]
\centering
\setlength{\parskip}{0pt}
\setlength{\abovecaptionskip}{-10pt}
\includegraphics[width=\linewidth]{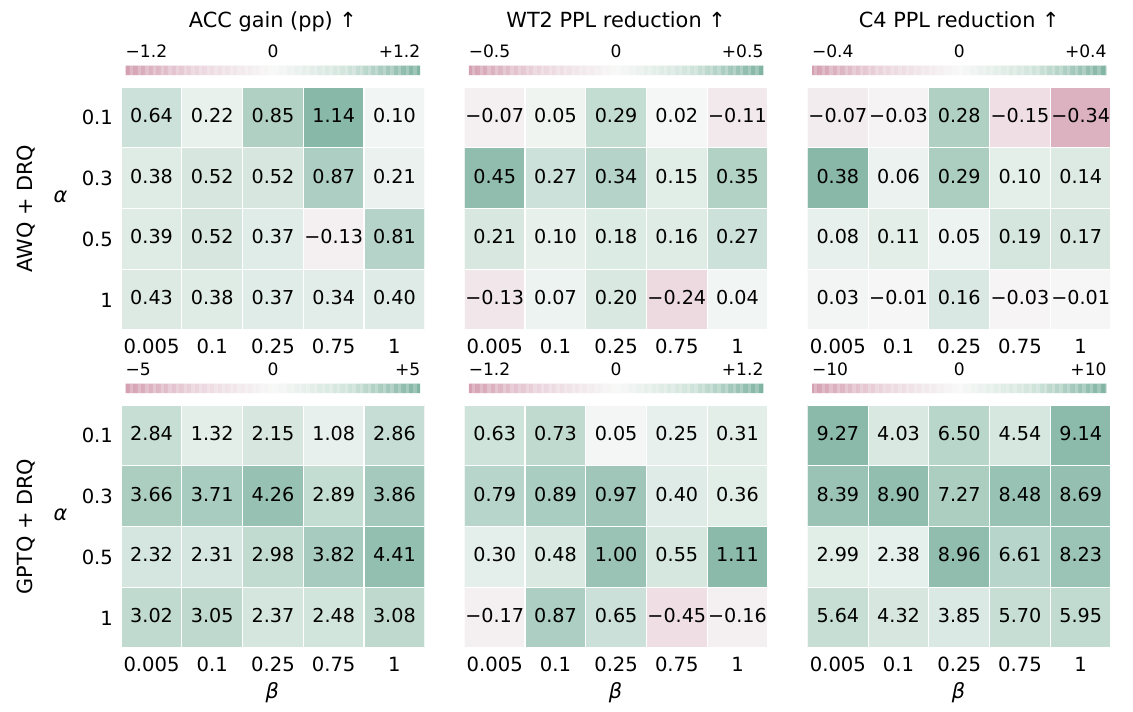}
\caption{Performance across different $\alpha$ and $\beta$ combinations on Llama-3.2-3B-Instruct at W3.}
\label{fig:parameter-sensitivity}
\end{figure}

\clearpage

%% file: tex/appendix/E_00_efficiency.tex
\section{Computational Efficiency}
\label{app:efficiency}
\label{app:calibration-details}

\input{tex/appendix/E_01_quantization_cost}
\input{tex/appendix/E_02_scaling}
\input{tex/appendix/E_03_search_complexity}
\input{tex/appendix/E_04_inference_overhead}

%% file: tex/appendix/E_01_quantization_cost.tex
\subsection{Refinement time across PTQ methods}
\label{app:cross-host-cost}

\settingrevision{\Cref{tab:cross-host-cost} compares DRQ refinement time for the settings in Appendix~\ref{app:cross-host-results}. The layer-refinement time sums 196 DRQ calls, one per linear layer. The time including integration also covers the surrounding code that calls DRQ within the base PTQ method or passes stored calibration activations to it. Both columns exclude base-method quantization and hyperparameter search. RTN, GPTQ, AWQ, OmniQuant, and GPTAQ were remeasured on the same A800 80GB GPU. ParoQuant uses its earlier timing measurement.}

\begin{table}[!htbp]
\centering
\input{tex/tables/cost.tex}
\end{table}

\FloatBarrier

%% file: tex/tables/cost.tex
\begingroup
\definecolor{ResultBlue}{HTML}{176B91}
\definecolor{ResultInk}{HTML}{202B35}

\caption{\langrevision{DRQ refinement time for Llama-3.2-3B W3.}}
\label{tab:cross-host-cost}
\fontsize{9.5}{11.5}\selectfont
\setlength{\tabcolsep}{3pt}
\renewcommand{\arraystretch}{1.18}
\begin{tabular*}{\linewidth}{@{\extracolsep{\fill}}lcc@{}}
\toprule
\langrevision{Base PTQ method} & \langrevision{Layer refinement (s)} & \langrevision{Including integration (s)} \\
\midrule
RTN & 556.74 & 788.04 \\
GPTQ & 227.28 & 233.24 \\
AWQ & 482.56 & 488.99 \\
OmniQuant & 520.38 & 792.13 \\
GPTAQ & 296.92 & 967.00 \\
ParoQuant (earlier run) & 220.79 & 794.04 \\
\bottomrule
\end{tabular*}
\endgroup

%% file: tex/appendix/E_02_scaling.tex
\FloatBarrier
\subsection{Scaling with model size and bit width}
\label{app:efficiency-scaling}

\Cref{fig:marq-scaling} compares DRQ refinement time at W2--W4 across 1B--70B Llama models. We sum the time spent refining linear layers. Each 1B--8B point averages the runs for that setting, with 86 runs in total at $(\alpha,\beta)=(0.3,0.25)$. The 1B--8B models use one A800 80GB GPU, and the 70B model uses two. At W3, refinement takes 2.24, 4.05, 5.95, and 28.05 minutes with GPTQ for 1B, 3B, 8B, and 70B, respectively. The corresponding AWQ times are 2.00, 3.90, 5.40, and 28.96 minutes. The two 70B W3 runs therefore spend similar time on DRQ. At 70B, W2/W4 take 36.55/26.85 minutes with GPTQ and 76.03/28.20 minutes with AWQ. The AWQ W2 and W4 runs use $(\alpha,\beta)=(1,1)$ and $(0.3,1)$, respectively. For every 1B--70B model and base PTQ method, W2 takes longer than W3 and W4, including when comparing runs with the same seeds. Refinement time depends on the code replacements tested and the number of search and checking steps performed.

\begin{figure}[!htbp]
\centering
\setlength{\parskip}{0pt}
\setlength{\abovecaptionskip}{-8pt}
\includegraphics[width=\linewidth]{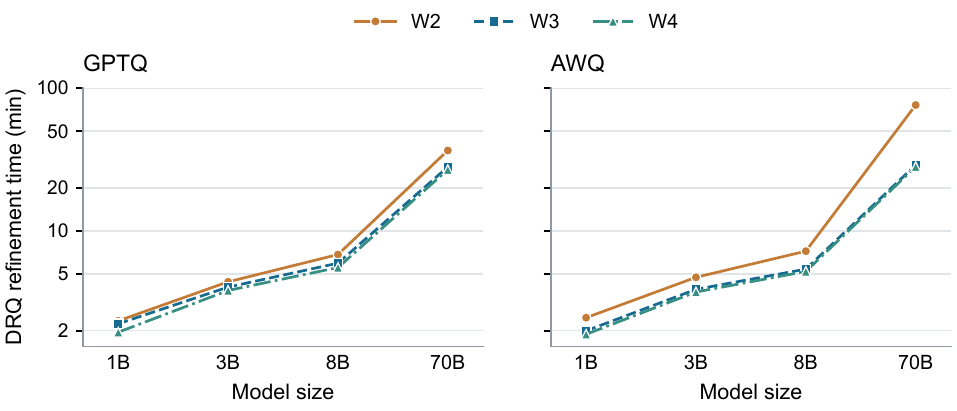}
\caption{DRQ refinement time across model sizes and weight bit widths.}
\label{fig:marq-scaling}
\end{figure}

\FloatBarrier

%% file: tex/appendix/E_03_search_complexity.tex
\FloatBarrier
\Needspace{5\baselineskip}
\subsection{Search complexity and deployment}
\label{app:offline-cost}
\label{app:complexity}

DRQ limits search cost by retaining a small set of integer-code changes and testing selected combinations. Consider a weight matrix with $m$ rows, one per output channel, and $d$ columns, one per input channel. With $n_{\rm d}$ calibration input vectors in the search subset and $B=|\mathcal Z_b|$ available integer codes, scoring all single-code changes costs $O(mdn_{\rm d}+mdB)$ per round. Only the $k=8$ shortlisted changes in each output row enter pair scoring, giving at most $\binom{k}{2}=28$ candidate pairs per row. Computing their activation inner products costs $O(mk^2n_{\rm d})$. This shortlist limits the cost of scoring pairs even for wide linear layers while accounting for interactions that single-code scores miss.

DRQ retains at most one update per output row, replacing one or two integer codes. It ranks these row updates by the decrease in the objective with the input direction $u$ held fixed in $O(m\log m)$ time and checks $p\le6$ distinct nonzero prefixes. Each prefix applies the first part of this ranked list to the current weights. Increasing its length changes more rows. Direct checks of reconstruction loss between quantized and full-precision layer outputs cost $O(pmd(n_{\rm d}+n_{\rm c}))$, where $n_{\rm c}$ counts all input vectors in the calibration subsets reserved for checking reconstruction loss. Wider layers and larger activation subsets make each check more expensive, but the prefix limit bounds how many checks are needed.

We also bound the number of numerical solves used to estimate the worst-case input direction. Let $T_{\rm eig}$ denote the time for one complete numerical solve, including matrix products, linear solves involving $\bC$, Krylov basis operations, and Rayleigh--Ritz extraction. Evaluating the initial quantized weights, the current weights at each round, and candidate prefixes costs at most $[1+\sum_{r=1}^{R}(1+p_{{\rm e},r})]T_{\rm eig}$ over $R\le3$ rounds, where $p_{{\rm e},r}\le p$ counts candidate solves in round $r$. Even if all six prefixes reach this stage in every round, at most 22 solves are needed per linear layer. Early rejection and stopping reduce this number. Initial reconstruction checks and the optional final comparison with the initial quantized weights in \cref{eq:exact-robust-excess} add one-time costs.

All of this computation takes place during offline refinement. DRQ changes the integer codes representing quantized weights while preserving the bit width, quantization grid, dequantization rule, and parameter layout. It therefore adds no inference operators or additional stored quantization parameters. Appendix~\ref{app:inference-overhead} measures the resulting latency, throughput, and memory use.

%% file: tex/appendix/E_04_inference_overhead.tex
\subsection{Inference overhead}
\label{app:inference-overhead}

We benchmark Llama-3.1-8B-Instruct at W4A16 with a group size of 128 on one NVIDIA A800 80GB GPU, using vLLM 0.6.3 in eager mode. Batch sizes are $\{1,4,8\}$, input lengths are $\{128,2048\}$, and output length is 128 tokens. We evaluate each initial quantized model and its DRQ-refined counterpart on the same workload and physical GPU, with 10 warm-up runs and 30 timed repeats per round. We time GPU-synchronized engine steps, including scheduling and sampling. Prefill latency covers the first step that processes the input sequence, while decoding latency averages the remaining 127 token-generation steps. Throughput accounts for both prefill and decoding. For latency and throughput, we take the median within each round and average these medians over four retained pairs of benchmark rounds for GPTQ and AWQ. Peak memory is the maximum PyTorch allocated memory across retained rounds.

\Cref{tab:inference-overhead-workloads} reports all six workloads. Across these workloads, prefill and decoding times change by $-0.80\%$ to $+0.82\%$ after refinement, computed from the geometric mean of paired latency ratios across rounds. Model tensors occupy the same amount of memory. DRQ changes the integer codes representing weights while preserving tensor layouts and inference operators, which explains the unchanged memory footprint under matched workloads. These results demonstrate that DRQ adds no extra inference computation or memory overhead.

\begin{table}[!htbp]
\centering
\input{tex/tables/inference_overhead_workloads}
\end{table}

\FloatBarrier

%% file: tex/tables/inference_overhead_workloads.tex
\begingroup
\definecolor{InferenceMARQBlue}{HTML}{176B91}
\small
\setlength{\tabcolsep}{2.5pt}
\renewcommand{\arraystretch}{1.10}
\caption{Inference cost across batch sizes and input lengths on Llama-3.1-8B-Instruct, W4A16. Output length is fixed at 128 tokens.}
\label{tab:inference-overhead-workloads}
\begin{tabular*}{\linewidth}{@{\extracolsep{\fill}}lcccccc@{}}
\toprule
Method & Batch size & \shortstack{Input\\(tokens)} & \shortstack{Prefill\\(ms)} & \shortstack{Decoding\\(ms/token)} & \shortstack{Throughput\\(tokens/s)} & \shortstack{Peak memory\\(GiB)} \\
\midrule
GPTQ & \multirow{2}{*}{1} & \multirow{2}{*}{128} & 33.61 & 16.52 & 60.04 & 7.86 \\
\textcolor{InferenceMARQBlue}{\quad + DRQ} &  &  & 33.49 & 16.49 & 60.16 & 7.86 \\
\addlinespace[2pt]
GPTQ & \multirow{2}{*}{1} & \multirow{2}{*}{2,048} & 156.58 & 16.49 & 56.92 & 8.01 \\
\textcolor{InferenceMARQBlue}{\quad + DRQ} &  &  & 157.86 & 16.59 & 56.58 & 8.01 \\
\addlinespace[2pt]
GPTQ & \multirow{2}{*}{4} & \multirow{2}{*}{128} & 49.07 & 17.04 & 231.35 & 7.89 \\
\textcolor{InferenceMARQBlue}{\quad + DRQ} &  &  & 49.05 & 16.90 & 233.17 & 7.89 \\
\addlinespace[2pt]
GPTQ & \multirow{2}{*}{4} & \multirow{2}{*}{2,048} & 520.06 & 17.06 & 190.56 & 8.65 \\
\textcolor{InferenceMARQBlue}{\quad + DRQ} &  &  & 520.16 & 17.15 & 189.75 & 8.65 \\
\addlinespace[2pt]
GPTQ & \multirow{2}{*}{8} & \multirow{2}{*}{128} & 77.81 & 17.24 & 451.75 & 7.93 \\
\textcolor{InferenceMARQBlue}{\quad + DRQ} &  &  & 77.80 & 17.17 & 453.57 & 7.93 \\
\addlinespace[2pt]
GPTQ & \multirow{2}{*}{8} & \multirow{2}{*}{2,048} & 1009.35 & 17.40 & 318.06 & 9.56 \\
\textcolor{InferenceMARQBlue}{\quad + DRQ} &  &  & 1009.24 & 17.41 & 318.04 & 9.56 \\
\midrule
AWQ & \multirow{2}{*}{1} & \multirow{2}{*}{128} & 36.58 & 18.25 & 54.38 & 7.70 \\
\textcolor{InferenceMARQBlue}{\quad + DRQ} &  &  & 36.56 & 18.15 & 54.67 & 7.70 \\
\addlinespace[2pt]
AWQ & \multirow{2}{*}{1} & \multirow{2}{*}{2,048} & 168.47 & 18.19 & 51.66 & 8.01 \\
\textcolor{InferenceMARQBlue}{\quad + DRQ} &  &  & 168.44 & 18.20 & 51.61 & 8.01 \\
\addlinespace[2pt]
AWQ & \multirow{2}{*}{4} & \multirow{2}{*}{128} & 71.86 & 18.72 & 209.06 & 7.89 \\
\textcolor{InferenceMARQBlue}{\quad + DRQ} &  &  & 71.88 & 18.74 & 208.87 & 7.89 \\
\addlinespace[2pt]
AWQ & \multirow{2}{*}{4} & \multirow{2}{*}{2,048} & 557.16 & 18.75 & 174.22 & 8.65 \\
\textcolor{InferenceMARQBlue}{\quad + DRQ} &  &  & 557.26 & 18.69 & 174.71 & 8.65 \\
\addlinespace[2pt]
AWQ & \multirow{2}{*}{8} & \multirow{2}{*}{128} & 103.41 & 18.92 & 408.67 & 7.93 \\
\textcolor{InferenceMARQBlue}{\quad + DRQ} &  &  & 103.44 & 18.84 & 410.28 & 7.93 \\
\addlinespace[2pt]
AWQ & \multirow{2}{*}{8} & \multirow{2}{*}{2,048} & 1074.83 & 19.02 & 293.35 & 9.55 \\
\textcolor{InferenceMARQBlue}{\quad + DRQ} &  &  & 1074.85 & 18.96 & 293.96 & 9.55 \\
\bottomrule
\end{tabular*}
\endgroup